\pdfoutput=1

\documentclass[nohyperref]{article}

\usepackage{microtype}
\usepackage{graphicx}
\usepackage{subfigure}
\usepackage{booktabs} 
\usepackage{multicol}

\usepackage{hyperref}

\usepackage[accepted]{icml2022}

\usepackage{amsmath}
\usepackage{amssymb}
\usepackage{mathtools}
\usepackage{amsthm}

\usepackage[capitalize,noabbrev]{cleveref}

\theoremstyle{plain}
\newtheorem{theorem}{Theorem}[section]

\newtheorem{lemma}[theorem]{Lemma}

\theoremstyle{definition}
\newtheorem{definition}[theorem]{Definition}
\newtheorem{assumption}[theorem]{Assumption}
\theoremstyle{remark}

\newtheorem{claim}[theorem]{Claim}

\usepackage[textsize=tiny]{todonotes}
\usepackage{graphicx}
\usepackage{mathtools}
\usepackage{footnote}
\usepackage{float}
\usepackage{xspace}
\usepackage{multirow}
\usepackage{wrapfig}
\usepackage{framed}
\usepackage[export]{adjustbox}
\usepackage{footnote}

\makesavenoteenv{tabular}
\makesavenoteenv{table}

\newcommand{\hDelta}{{\hat{\Delta}}}

\usepackage{amsmath}
\usepackage{amsfonts}
\usepackage{amssymb}
\usepackage{amsthm}
\usepackage{bm}
\usepackage{bbm}
\usepackage{mathtools}
\usepackage{enumitem}
\usepackage{thmtools,thm-restate}
\usepackage{algorithm}
\usepackage{algorithmic}
\usepackage{graphicx}
\usepackage{comment}

\newcommand{\kevin}[1]{{\color{magenta} KGJ: #1}}

\newcommand{\algoname}{\texttt{Active Task Relevance Sampling}}

\newenvironment{proofsketch}{\textit{Proof sketch:}}{\qed \par}

\def\Ebb{\mathbb{E}}

\def\Rbb{\mathbb{R}}

\def\R{\Rbb}

\def\*{\star}

\DeclareMathOperator*{\argmin}{arg\,min}

\newcommand{\E}{\Ebb}

\usepackage{graphicx}
\usepackage{mathtools}
\usepackage{footnote}
\usepackage{float}
\usepackage{xspace}
\usepackage{multirow}
\usepackage{wrapfig}
\usepackage{framed}
\usepackage{xcolor}

\newcommand{\field}[1]{\mathbb{#1}}

\newcommand{\fR}{\field{R}}

\newcommand{\calE}{{\mathcal{E}}}

\newcommand{\calX}{{\mathcal{X}}}
\newcommand{\calZ}{{\mathcal{Z}}}
\newcommand{\calY}{{\mathcal{Y}}}

\newcommand{\calN}{{\mathcal{N}}}

\newcommand{\prob}{\text{Prob}}

\newcommand{\order}{\ensuremath{\mathcal{O}}}
\newcommand{\otil}{\ensuremath{\widetilde{\mathcal{O}}}}

\newcommand{\one}{\boldsymbol{1}}

\newcommand{\hB}{\hat{B}}
\newcommand{\hW}{\hat{W}}
\newcommand{\hw}{\hat{w}}
\newcommand{\hSigma}{\hat{\Sigma}}

\icmltitlerunning{
Active Multi-Task Representation Learning
}

\begin{document}

\twocolumn[
\icmltitle{
Active Multi-Task Representation Learning
}



\icmlsetsymbol{equal}{*}

\begin{icmlauthorlist}
\icmlauthor{Yifang Chen}{xx}
\icmlauthor{Simon S. Du}{xx}
\icmlauthor{Kevin Jamieson}{xx}
\end{icmlauthorlist}

\icmlaffiliation{xx}{Paul G. Allen School of Computer Science \& Engineering,
University of Washington}

\icmlcorrespondingauthor{Yifang Chen}{yifangc@cs.washington.edu}

\icmlkeywords{Machine Learning, ICML}

\vskip 0.3in
]



\printAffiliationsAndNotice{}  
\begin{abstract}

To leverage the power of big data from source tasks and overcome the scarcity of the target task samples, representation learning based on multi-task pretraining has become a standard approach in many applications.
However, up until now, choosing which source tasks to include in the multi-task learning has been more art than science.
In this paper, we give the first formal study on resource task sampling by leveraging the techniques from active learning.
We propose an algorithm that iteratively estimates the relevance of each source task to the target task and samples from each source task based on the estimated relevance.
Theoretically, we show that for the linear representation class, to achieve the same error rate, our algorithm can save up to a \emph{number of source tasks} factor in the source task sample complexity, compared with the naive uniform sampling from all source tasks.
We also provide experiments on real-world computer vision datasets to illustrate the effectiveness of our proposed method on both linear and convolutional neural network representation classes. 
We believe our paper serves as an important initial step to bring techniques from active learning to representation learning.

\end{abstract}
\section{Introduction}

Much of the success of deep learning is due to its ability to efficiently learn a map from high-dimensional, highly-structured input like natural images into a dense, relatively low-dimensional representation that captures the semantic information of the input.
Multi-task learning leverages the observation that similar tasks may share a common representation to train a single representation to overcome a scarcity of data for any one task. 
In particular, given only a small amount of data for a target task, but copious amounts of data from source tasks, the source tasks can be used to learn a high-quality low-dimensional representation, and the target task just needs to learn the map from this low-dimensional representation to its target-specific output. 
This paradigm has been used with great success in natural language processing domains 
GPT-2 \cite{radford2019language}, GPT-3 \cite{brown2020language}, Bert \cite{devlin2018bert}, as well as vision domains CLIP \cite{radford2021learning}. 

This paper makes the observation that not all tasks are equally helpful for learning a representation, and a priori, it can be unclear which tasks will be best suited to maximize performance on the target task.
For example, modern datasets like CIFAR-10, ImageNet, and the CLIP dataset were created using a list of search terms and a variety of different sources like search engines, news websites, and Wikipedia. \cite{Krizhevsky09learningmultiple,5206848,radford2021learning}
Even if more data always leads to better performance, practicalities demand some finite limit on the size of the dataset that will be used for training.
Up until now, choosing which source tasks to include in multi-task learning has been an ad hoc process and more art than science. 
\emph{In this paper, we aim to formalize the process of prioritizing source tasks for representation learning by formulating it as an active learning problem.}


Specifically, we aim to achieve a target accuracy on a target task by requesting as little total data from source tasks as possible. 
For example, if a target task was to generate captions for images in a particular domain where few examples existed, each source task could be represented as a search term into Wikipedia from which (image, caption) pairs are returned.
By sampling moderate numbers of (image, caption) pairs resulting from each search term (task), we can determine which tasks result in the best performance on the source task and increase the rate at which examples from those terms are sampled.   
By quickly identifying which source tasks are useful for the target task and sampling only from those, we can reduce the overall number of examples to train over, potentially saving time and money. 
Moreover, prioritizing relevant tasks in training, in contrast to uniformly weighting them, even has the potential to improve performance, as demonstrated in \cite{chen2021weighted}.



From the theoretical perspective, \citet{tripuraneni2020theory,tripuraneni2021provable,du2020few} study few-shots learning via multi-task representation learning and gives the generalization guarantees, that, such representation learning can largely reduce the target sample complexity. But all those works only consider \textit{uniform sampling} from each source task and thus establish the proof based on benign diversity assumptions on the sources tasks as well as some common assumptions between target and source tasks.

In this paper, we initiate the systematic study on using active learning to sample from source tasks.
We aim to achieve the following two goals:
\begin{enumerate}
    \item If there is a fixed budget on the source task data to use during training, we would like to select sources that maximize the accuracy of target task relative to naive uniform sampling from all source tasks. Equivalently, to achieve a given error rate, we want to reduce the amount of required source data. In this way, we can reduce the computation because the training complexity generally scales with the amount of data used, especially when the user has limited computing resources (e.g., a finite number of GPUs).
    \item  Given a target task, we want to output a relevance score for each source task, which can be useful in at least two aspects. 
    First, the scores suggest which certain source tasks are helpful for the target task and inform future task or feature selection (sometimes the task itself can be regard as some latent feature).
    Second, the scores help the user to decide which tasks to sample more, in order to further improve the target task accuracy.
\end{enumerate}

\subsection{Our contributions} 
In our paper, given a single target task and $M$ source tasks we propose a novel quantity $\nu^* \in \R^M$ that characterizes the relevance of each source task to the target task (cf. Defn~\ref{defn:nu_star}).
We design an active learning algorithm which can take any representation function class as input.
The algorithm iteratively estimates $\nu^*$ and samples data from each source task based on the estimated $\nu^*$.
The specific contributions are summarized below:
\begin{itemize}
    \item In Section~\ref{sec: warm up}, we give the definition of $\nu^*$. As a warm up, we prove that when the representation function class is linear and $\nu^*$ is known, if we sample data from source tasks according to the given $\nu^*$, the sample complexity of the source tasks scales with the sparsity of $\nu^* \in \R^M$ (the $m$-th task is relevant if $\nu_m^* \neq 0$). This can save up to a factor of $M$, the \emph{number of source tasks}, compared with the naive uniform sampling from all source tasks.
    \item In Section~\ref{sec: main algo}, we drop the assumption of knowing $\nu^*$ and describe our active learning algorithm that iteratively samples examples from tasks to estimate $\nu^*$ from data.
    We prove that when the representation function class is linear, our algorithm never performs worse than uniform sampling, and achieves a sample complexity nearly as good as when $\nu^*$ is known. The key technical innovation here is to have a trade-off on less related source tasks between saving sample complexity and collecting sufficient informative data for estimating $\nu^*$. 
    \item  In Section~\ref{sec: experiment}, we empirically demonstrate the effectiveness of our active learning algorithm by testing it on the corrupted MNIST dataset with both linear and convolutional neural network (CNN) representation function classes. 
    The experiments show our algorithm gains substantial improvements compared to the non-adaptive algorithm on both models.
    Furthermore, we also observe that our algorithm generally outputs higher relevance scores for source tasks that are semantically similar to the target task.
\end{itemize}

\subsection{Related work} 

There are many existing works on provable \emph{non-adaptive} representation learning with various assumptions. \citet{tripuraneni2020theory,tripuraneni2021provable,du2020few,thekumparampil2021sample,collins2021exploiting,xu2021representation} assume there exists an underlying representation shared across all tasks. (Notice that some works focus on learning a representation function for \textit{any possible target task}, instead of learning a model for a specific target task as is the case in our work.) In particular, \citet{tripuraneni2020theory,thekumparampil2021sample} assume a low dimension linear representation. Furthermore, it assumes the  covariance matrix of all input features is the identity and the linear representation model is orthonormal. 
\citet{du2020few,collins2021exploiting} also study a similar setting but lift the identity  covariance and orthonormal assumptions.
Both works obtain similar conclusions. We will discuss our results in the context of these two settings in Section~\ref{sec: prelim}.

Going beyond the linear representation, \citet{du2020few}  generalize their bound to a 2-layer ReLu network and \citet{tripuraneni2021provable} further considers any general representation and linear predictor classes. 
More recent work has studied fine-tuning in both theoretical and empirical contexts \citet{shachaf2021theoretical,chua2021finetuning,chen2021weighted}. 
We leave extending our theoretical analysis to more general representation function classes as future work. Other than the generalization perspective, \citet{tripuraneni2021provable,thekumparampil2021sample,collins2021exploiting} propose computational efficient algorithms in solving this non-convex empirical minimization problems during representation learning, including Method-of-moments (MOM) algorithm and Alternating Minimization. Incorporating these efficient algorithms into our framework would also be a possible direction in the future.

\citet{chen2021weighted} also consider learning a weighting over tasks.
However, their motivations are much different since they are working under the hypothesis that some tasks are not only irrelevant, but even \emph{harmful} to include in the training of a representation. 
Thus, during training they aim to down-weight potentially harmful source tasks and up-weight those source tasks most relevant to the target task.
But the critical difference between their work and ours is that they assume a pass over the complete datasets from all tasks is feasible whereas we assume it is not (e.g., where each task is represented by a search term to Wikipedia or Google). 
In our paper, their setting would amount to being able to solve for $\nu^*$ for free, the equivalent of the ``known $\nu^*$'' setting of our warm-up section.
However, our main contribution is an active learning algorithm that ideally only looks at a vanishing fraction of the data from all the sources to train a representation. 

There exists some empirical multi-task representation learning/transfer learning works that have similar motivations as us. For example, \citet{yao2021nlp} use a heuristic retriever method to select a subset of target-related NLP source tasks and show training on a small subset of source tasks can achieve similar performance as large-scale training. \citet{zamir2018taskonomy,devlin2018bert} propose a transfer learning algorithm based on learning the underlying structure among visual tasks, which they called Taskonomy, and gain substantial experimental improvements. 

Many classification, regression, and even optimization tasks may fall under the umbrella term \emph{active learning} \cite{settles2009active}.
We use it in this paper to emphasize that a priori, it is unknown which source tasks are relevant to the target task. We overcome this challenge by iterating the closed-loop learning paradigm of 1) collect a small amount of data, 2) make inferences about task relevancy, and 3) leverage these inferences to return to 1) with a more informed strategy for data collection. 





\section{Preliminaries}
\label{sec: prelim}
In this section, we formally describe our problem setup which will be helpful for our theoretical development.
\paragraph{Problem setup.}
Suppose we have $M$ source tasks and one target task, which we will denote as task $M+1$. Each task $m \in [M+1]$ is associated with a joint distribution $\mu_m$ over $\calX \times \calY$, where $\calX \in \fR^d$ is the input space and $\calY \in \fR$ is the output space. We assume there exists an underlying representation function $\phi^*: \calX \to \calZ$ that maps the input to some feature space $\calZ \in \fR^K$ where $K \ll d$. 
We restrict the representation function to be in some function class $\Phi$, e.g., linear functions, convolutional nets, etc. We also assume the linear predictor to be a linear mapping from feature space to output space, which is represented by $w_m^* \in \fR^K$. 
Specifically, we assume that for each task $m \in [M+1]$, an i.i.d sample $(x,y) \sim \mu_m$ can be represented as $y = \phi(x)^\top w_m^* + z$, where $z \sim \calN(0,\sigma^2)$.
Lastly, we also impose a regularity condition such that for all $m$, the distribution of $x$ when $(x,y) \sim \mu_m$ is $1$-sub-Gaussian.

During the learning process, we assume that we have only a small, fixed amount of data $\{x_{M+1}^i,y_{M+1}^i\}_{i \in [n_{M+1}]}$ drawn i.i.d. from the target task distribution $\mu_{M+1}$.
On the other hand, at any point during learning we assume we can obtain an i.i.d. sample from any source task $m \in [M]$ without limit. 
This setting aligns with our main motivation for active representation learning where we usually have a limited sample budget for the target task but nearly unlimited access to large-scale source tasks (such as (image,caption) example pairs returned by a search engine from a task keyword).  

Our goal is to use as few total samples from the source tasks as possible to learn a representation and linear predictor $\phi,w_{M+1}$ that minimizes the excess risk on the target task defined as
\begin{align*}
    &\text{ER}_{M+1}(\phi,w) 
    & = L_{M+1}(\phi,w) - L_{M+1}(\phi^*,w_{M+1}^*) 
\end{align*}
where $L_{M+1}(\phi,w) = \E_{(x,y) \sim \mu_{M+1}} \left[\left(\langle \phi(x), w\rangle - y\right)^2\right]$.

Our theoretical study focuses on the linear representation function class, which is studied in ~\citep{du2020few,tripuraneni2020theory,tripuraneni2021provable,thekumparampil2021sample}.

\begin{definition}[low-dimension linear representation]
\label{assump: linear model}
    $\Phi = \{x \to B^\top x \mid B \in \fR^{d \times K} \}$. We denote the true underlying representation function as $B^*$. Without loss of generality, we assume for all $m \in [M+1]$, $\E_{\mu_m)}[xx^\top]$ are equal.
\end{definition}
We also make the following assumption which has been used in \cite{tripuraneni2020theory}.
We note that Theorem~\ref{thm: warm-up} does not require this assumption,  but Theorem~\ref{thm: main} does.
\begin{assumption}[Benign low-dimension linear representation]
\label{assump: identity covar}
We assume $\E_{\mu_m}[x x^\top] = I$ and $\Omega(1) \leq \|w_m^*\|_2 \leq R$ for all $m \in [M+1]$. We also assume $B^*$ is not only linear, but also orthonormal.
\end{assumption}


\paragraph{Notations}
We denote the $n_m$ i.i.d samples collected from source task $m$ as the input matrix $X_m \in \fR^{n_m \times d}$, output vector $Y_m \in \fR^{n_m}$ and noise vector $Z_m \in \fR^{n_m}$. 
We then denote the expected and empirical input variances as $\Sigma_m = \E_{(x,y) \sim \mu_m} x x^\top$ 
and $\hat{\Sigma}_m = \frac{1}{n_m}(X_m)^\top X_m$. 
In addition, we denote the collection of $\{w_m\}_{m \in [M]}$ as $W \in \fR^{K \times M}$. Note that, the learning process will be divided into several epochs in our algorithm stated later, so we sometimes add subscript or superscript $i$ on those empirical notations to refer to the data used in certain epoch $i$. Finally, we use $\otil$ to hide $\log(K,M,d,1/\varepsilon, \sum_{m=1}^M n_m)$.

\paragraph{Other data assumptions}
Based on our large-scale source tasks motivation, we assume $M \geq K$ and $\sigma_{\min}(W^*) > 0$, which means the source tasks are diversified enough to learn all relevant representation features with respect to the low-dimension space.
This is the standard diversity assumption used in many recent works~\cite{du2020few,tripuraneni2020theory,tripuraneni2021provable,thekumparampil2021sample}. In addition, we assume $\sigma \geq \Omega(1)$ to make our main result easier to read. This assumption can be lift by adding some corner case analysis.

\section{Task Relevance $\nu^*$ and More Efficient Sampling with Known $\nu*$}
\label{sec: warm up}

In this section, we give our key definition of task relevance, based on which, we design a more efficient source task sampling strategy.

Note because $\sigma_{\min}(W^*) > 0$, we can regard $w_{M+1}^*$ as a linear combination of $\{w_{m}^*\}_{m \in [M]}$. 
\begin{definition}
\label{defn:nu_star}
$\nu^*\in \fR^M$ is defined as
\begin{align}
    &\nu^* = \argmin_{\nu} \|\nu\|_2
    & \text{s.t.} \quad 
    W^* \nu = w_{M+1}^* \label{eq:nu_star}
\end{align}
\end{definition}
where larger $|\nu^*(m)|$ means higher relevance between source task $m$ and the target task. If $\nu^*$ is known to the learner, intuitively, it makes sense to draw more samples from source tasks that are most relevant.

For each source task $m \in [M]$, Line~3 in Alg.~\ref{alg:known} draws $n_m \propto (\nu^*(m))^2$ samples.
The algorithm then estimates the shared representation $\phi: \R^d \rightarrow \R^K$, and task-specific linear predictors $W = \{w_m^*\}_{m=1}^M$ by empirical risk minimization across all source tasks following the standard multi-task representation learning approach. 

\begin{algorithm}[!t] 
\caption{Multi-task sampling strategy with Known $\nu^*$\label{alg:known}
}
\begin{algorithmic}[1] 
\label{alg:main_1}
\STATE \textbf{Input:} confidence $\delta$, representation function class $\Phi$, combinatorial coefficient $\nu^*$, source-task sampling budget $N_\text{total} \gg M(Kd + \log(M/\delta))$
\STATE Initialize the lower bound $\underline{N} = Kd+\log(M/\delta)$ and  number of samples $n_m = \max\left\{ (N_\text{total}-M\underline{N})\frac{(\nu^*(m))^2}{\|\nu^*\|_2^2}, \underline{N}\right\}$ for all $m \in [M]$.
\STATE For each task $m$, draw $n_m$ i.i.d samples from the corresponding offline dataset denoted as $\{X_m,Y_m\}_{m=1}^M$ 
\STATE Estimate the models as
\vspace{-10pt}
\begin{align}
\label{eqn: ERM 1}
&\hat{\phi}, \hat{W} = \argmin_{\phi \in \Phi,{W}=[{w}_1,\ldots,w_M]} \sum_{m=1}^M \| \phi(X_m) {w}_m - Y_m\|^2.\\
\label{eqn: ERM 2}
&\hat{w}_{M+1} = \argmin_{w} \|\hat{\phi}(X_{M+1}) w - Y_{M+1}\|^2
\end{align}
\vspace{-15pt}
\STATE \text{Return}  $\hat{\phi}, \hat{w}_{M+1}$
\end{algorithmic}
\end{algorithm} 

Below, we give our theoretical guarantee on the sample complexity from the source tasks when $\nu^*$ is known.


\begin{theorem}
\label{thm: warm-up}
Under the low-dimension linear representation setting as defined in Definition~\ref{assump: linear model}, with probability at least $1-\delta$, our algorithm's output satisfies $\textsc{ER}(\hat{B},\hat{w}_{M+1}) \leq \varepsilon^2$ whenever the total sampling budget from all sources $N_{total}$ is at least
\begin{align*}
    \otil\left((Kd+KM + \log(1/\delta)) \sigma^2 s^* \|\nu^*\|_2^2 \varepsilon^{-2} \right) 
\end{align*}
and the number of target samples $n_{M+1}$ is at least
\begin{align*}
    \otil(\sigma^2(K + \log(1/\delta))\varepsilon^{-2})
\end{align*}
where $s^* = \min_{\gamma \in [0,1]} (1- \gamma) \| \nu^* \|_{0,\gamma} + \gamma M$ and 
$\| \nu \|_{0,\gamma} := \left| \left\{ m : |\nu_m| > \sqrt{ \gamma \frac{\|\nu^*\|_2^2}{N_{total}} } \right\} \right|$.
\end{theorem}

Note that the number of target samples $n_{M+1}$ scales only with the dimension of the feature space $K$, and \emph{not} the input dimension $d \gg K$ which would be necessary without multi-task learning. 
This dependence is known to be optimal \cite{du2020few}.
The quantity $s^*$ characterizes our algorithm's ability to adapt to the approximate sparsity of $\nu^*$.  
Noting that $\sqrt{\frac{\|\nu^*\|_2^2}{N_{total}}}$ is roughly on the order of $\varepsilon$, taking $\gamma \approx 1/M$ suggests that to satisfy $\textsc{ER}(\hat{B},\hat{w}_{M+1}) \leq \varepsilon^2$, only those source tasks with relevance $|\nu^*(m)| \gtrapprox \varepsilon$ are important for learning.

For comparison, we rewrite  the bound in \cite{du2020few} in the form of $\nu^*$. 
\begin{theorem}
\label{thm: compare to uniform}
Under Assumption~\ref{assump: linear model}, to obtain the same accuracy result, the non-adaptive (uniform) sampling of \cite{du2020few} requires that the total sampling budget from all sources $N_\text{total}$ is at least
\begin{align*}
    \otil\left((Kd+KM + \log(1/\delta)) \sigma^2 M \|\nu^*\|_2^2 \varepsilon^{-2} \right) 
\end{align*}
and requires the same amount of target samples as above.
\end{theorem}
Note the key difference is that the $s^*$ in Theorem~\ref{thm: warm-up} is replaced by $M$ in Theorem~\ref{thm: compare to uniform}.
Below we give a concrete example to show this difference is significant.
\vspace{-10pt}
\paragraph{Example: Sparse $\nu^*$.}
Consider an extreme case where $w_m = e_{m\text{ mod }(K-1)+1}$ for all $m \in [M-1]$, and $w_M = w_{M+1} = e_K$. This suggests that the target task is exactly the same as the source task $M$ and all the other source tasks are uninformative. It follows that $\nu^*$ is a $1$-sparse vector $e_M$ and $s^*=1$  when $\gamma = 0$. We conclude that uniform sampling requires a sample complexity that is $M$ times larger than that of our non-uniform procedure. 


\subsection{Proof sketch of Theorem~\ref{thm: warm-up}}

We first claim two inequalities that are derived via straightforward modifications of the proofs in \citet{du2020few}:
\begin{align}
\label{eq: er1}
    &\textsc{ER}(\hat{B},\hat{w}_{M+1})
    \lessapprox \tfrac{\|P_{X_{M+1}\hat{B}}^{\bot}X_{M+1}B^* w_{M+1}^*\|^2}{n_{M+1}} \\
\label{eq: er2}
    & \tfrac{\|P_{X_{M+1}\hat{B}}^{\bot}X_{M+1}B^* \widetilde{W}^*\|_F^2}{n_{M+1}} 
    \lessapprox \sigma^2 \big(K(M+d) + \log\tfrac{1}{\delta} \big)
\end{align}
where $P_{A}^\perp=I - A\left(A^{\top} A\right)^{\dagger} A^{\top}$,  $\tilde{\nu}^*(m) = \frac{\nu^*(m)}{\sqrt{n_m}}$, and $\tilde{W}$ is $\left[\sqrt{n_1} w_1^*, \sqrt{n_2} w_2^*, \ldots, \sqrt{n_M} w_M^*\right]$. 
By using these two results and noting that $w_{M+1}^* = \widetilde{W}^*\tilde{\nu}^*$, we have
\begin{align*}
    \textsc{ER}(\hat{B},\hat{w}_{M+1})
    &\overset{\eqref{eq: er1}}{\lessapprox} \frac{1}{n_{M+1}}\|P_{X_{M+1}\hat{B}}^{\bot}X_{M+1}B^* \widetilde{W}^*\tilde{\nu}^*\|_2^2\\
    &\leq \frac{1}{n_{M+1}} \|P_{X_{M+1}\hat{B}}^{\bot}X_{M+1}B^* \widetilde{W}^*\|_F^2\|\tilde{\nu}^*\|_2^2\\
    & = \eqref{eq: er2} \times \|\tilde{\nu}^*\|_2^2 .
\end{align*}
The key step to our analysis is the decomposition of $\|\tilde{\nu}^*\|_2^2$. 
If we denote $\epsilon^{-2} = \frac{N_\text{total}}{\|\nu^*\|_2^2}$, we have, for any $\gamma \in [0,1]$,
\begin{align*}
    & \sum_m \frac{\nu^*(m)^2}{n_m}\left(\one\{ |\nu^*(m)| > \sqrt{\gamma}\epsilon\} + \one\{ |\nu^*(m)| \leq \sqrt{\gamma}\epsilon\}\right)\\
    & \lessapprox \sum_m \left( \epsilon^{2}\one\{ |\nu^*(m)| > \sqrt{\gamma}\epsilon\} + \gamma\epsilon^2\one\{ |\nu^*(m)| \leq \sqrt{\gamma}\epsilon\}\right)
\end{align*}
where the inequality comes from the definition of $n_m$ and the fact $N_\text{total} \gg M\underline{N} $.
Now by replacing the value of $\epsilon$ and $ \|\nu\|_{0,\gamma}$, we get the desired result.
\section{Main Algorithm and Theory}
\label{sec: main algo}

In the previous section, we showed the advantage of target-aware source task sampling when the optimal mixing vector $\nu^*$ between source tasks and the target task is \emph{known}. 
In practice, however, $\nu^*$ is unknown and needs to be estimated based on the estimation of $W^*$ and $w_{M+1}^*$, which are themselves consequences of the unknown representation $\phi^*$. 
In this section, we design an algorithm that adaptively samples from source tasks to efficiently learn $\nu^*$ and the prediction function for the target task $B^*w_{M+1}^*$. 
The pseudocode for the procedure is found in Alg.~\ref{alg:main_1}.

\begin{algorithm}[!t] 
\caption{\algoname} \label{alg:main_1}
\begin{algorithmic}[1] 
\STATE \textbf{Input:} confidence $\delta$, a lower bound of $\sigma_{\min}(W^*)$ as $\underline{\sigma}$, representation function class $\Phi$
\STATE Initialize $\hat{\nu}_1 = [1/M,1/M,\ldots]$, $\epsilon_i = 2^{-i}$ and
$\{\beta_i\}_{i=1,2,\ldots}$, which will be specified later
\FOR{$i=1,2, \ldots $}
    \STATE Set $n_m^i = \max\left\{\beta_i \hat{\nu}_{i}^2(m)\epsilon_i^{-2},\beta_i \epsilon_i^{-1}\right\}$. 
    \STATE For each task $m$, draw $n_m$ i.i.d samples from the corresponding offline dataset denoted as $\{X_m^i,Y_m^i\}_{m=1}^M$
    \STATE Estimate $\hat{\phi}^i, \hat{W}_i,\hat{w}_{M+1}^i$ with Eqn.~\eqref{eqn: ERM 1} and \eqref{eqn: ERM 2}
    \STATE Estimate the coefficient as
    \vspace{-10px}
    \begin{align}
    \label{eq: v hat}
        \hat{\nu}_{i+1} = \argmin_\nu \|\nu\|_2^2  
        \quad \text{s.t.} 
        \quad \hat{W}_i\nu = \hat{w}_{M+1}^i
    \end{align}
    \vspace{-20px}
\ENDFOR
\end{algorithmic}
\end{algorithm}


We divide the algorithm into several epochs. 
At the end of each epoch $i$, we obtain estimates $\hat{\phi}_i, \hat{W}_i$ and $\hat{w}_{M+1}^i$ which are then used to calculate the task relevance denoted as $\hat{\nu}_{i+1}$. Then in the next epoch $i+1$, we sample data based on $\hat{\nu}_{i+1}$.
The key challenge in this iterative estimation approach is that the error of the estimation propagates from round to round due to unknown $\nu^*$ if we directly apply the sampling strategy proposed in Section~\ref{sec: warm up}. To avoid inconsistent estimation, we enforce the condition that each source task is sampled at least $\beta \epsilon_i^{-1}$ times to guarantee that $|\hat{\nu}_{i}(m)|$ is always $\sqrt{\epsilon_i}$-close to $|c\nu^*(m)|$, where $c \in [1/16,4]$. We will show why such estimation is enough in our analysis.

\subsection{Theoretical results under linear representation}

Here we give a theoretical guarantee for the realizable linear representation function class. Under this setting, we choose
\begin{align*}
    \beta: &= \beta_i =  \bigg( 3000K^2R^2(KM+Kd\log(\frac{N_\text{total}}{\varepsilon M} )\\ & \hspace{.5in}+\log(\frac{M \log(1/N_\text{total})}{\delta/10}) \bigg)\frac{1}{\underline{\sigma}^6}, \forall i
\end{align*}

\begin{theorem}
\label{thm: main}
Suppose we know in advance a lower bound of $\sigma_{\min}(W^*)$ denoted as $\underline{\sigma}$.
Under the benign low-dimension linear representation setting as defined in Assumption~\ref{assump: identity covar}, we have $\textsc{ER}(\hat{B},\hat{w}_{M+1}) \leq \varepsilon^2$ with probability at least $1-\delta$ whenever the number of source samples $N_{total}$ is at least
\begin{align*}
     &\otil\bigg(\left(K(M+d) + \log\frac{1}{\delta} \right) \sigma^2  s^* \|\nu^*\|_2^2 \varepsilon^{-2} 
     + \square \sigma \varepsilon^{-1} \bigg)
\end{align*}
where $\square = \left(MK^2dR/\underline{\sigma}^3\right)\sqrt{s^*}$
and the target task sample complexity $n_{M+1}$ is at least
\begin{align*}
    &\otil\left(\sigma^{2}K\varepsilon^{-2} +\Diamond\sqrt{s^*} \sigma \varepsilon^{-1}\right)
\end{align*}
where $\Diamond = \min\left\{\frac{\sqrt{R}}{\underline{\sigma}^2K} ,\sqrt{K(M+d)+\log\frac{1}{\delta} }\right\}$ and
 $s^*$ has been defined in Theorem~\ref{thm: warm-up}. 
 
\end{theorem}

\paragraph{Discussion.}
Comparing to the known $\nu^*$ case studied in the previous section, in this unknown $\nu^*$ setting our algorithm only requires an additional low order term 
$\square\sigma\varepsilon^{-1}$ to achieve the same objective (under the additional assumption of Assumption~\ref{assump: identity covar}). 
%
%
Also, as long as $\Diamond \leq \otil(\sigma K \varepsilon^{-1})$, our target task sample complexity $\otil(\sigma^2K\varepsilon^{-2})$ remains the optimal rate \citep{du2020few}. 


Finally, we remark that a limitation of our algorithm is that it requires some prior knowledge of $\underline{\sigma}$. However, because it only hits the low-order $\epsilon^{-1}$ terms, this is unlikely to dominate either of the sample complexities for reasonable values of $d,K,$ and $M$. 
\subsection{Proof sketch}

\textbf{Step 1:} We first show that the estimated distribution over tasks $\hat{\nu}_i$ is close to the underlying $\nu^*$.

\begin{lemma}[Closeness between $\hat{\nu}_i$ and $\nu^*$]
\label{lem: nu estimation error(main)}
With probability at least $1-\delta$, for any $i$, as long as $n_{M+1} \geq \frac{ 2000\epsilon_i^{-1}}{ \underline{\sigma}^4} $, we have 
\begin{equation*}
    |\hat{\nu}_{i+1}(m)| \in
    \begin{cases}
       \left[|\nu^*(m)|/16, 4|\nu^*(m)|\right] & \text{if }  \nu^*(m) \geq \sigma \sqrt{\epsilon_i}\\
       [0, 4\sqrt{\epsilon_i}] & \text{if } |\nu^*(m)| \leq \sigma \sqrt{\epsilon_i}
    \end{cases}       
\end{equation*}
\end{lemma}

Notice that the sample lower bound in the algorithm immediately implies sufficiently good estimation in the next epoch even if $\hat{\nu}_{i+1}$ goes to $0$.

\begin{proofsketch}

Under Assumption~\ref{assump: identity covar}, by solving Eqn~\eqref{eqn: ERM 2}, we can rewrite the optimization problem on $\nu^*$ and $\hat{\nu}_i$ defined in Eqn.\eqref{eq:nu_star} and \eqref{eq: v hat} roughly as the follows (see the formal definition in the proof of Lemma~\ref{lem: nu estimation error(app)} in Appendix~\ref{sec: main analysis(app)})
\begin{align*}
    &\hat{\nu}_{i+1}
    =\argmin_\nu \|\nu\|_2^2  \\
    \text{s.t.} \quad
    & \sum_m  \hat{B}_i^\top \left(B^*w_m^* 
        +\frac{1}{n_m^i} \left(X_m^i\right)^\top Z_m\right)\nu(m) \\
       & =  \hat{B}_i^\top\left(B^*w_{M+1}^*  +\frac{1}{n_{M+1}^i} \left(X_{M+1}^i\right)^\top Z_{M+1}\right),\\
 \text{and}~   &\nu^* = \argmin_\nu \|\nu\|_2^2 \\
    \text{s.t.} \quad
    &  \sum_m w_m^*\nu(m) = w_{M+1}^*.
\end{align*}
Solving these two optimization problem gives,
\begin{align*}
     \nu^*(m)&= (B^* w_m^*)^T\left(B^*W^*(B^*W^*)^T\right)^{+} (B^*w_{M+1}^*) \\
    \lvert \hat{\nu}_{i+1}(m) \rvert & \leq 2\left\lvert(B^*w_m^*)^T (\hat{B}_i\hat{W}_i(\hat{B}_i\hat{W}_i)^T)^+  B^*w_{M+1}^* \right\lvert \\
    & \quad + \text{low order noise}.
\end{align*}
It is easy to see that the main different between these two expressions is $\left(B^*W^*(B^*W^*)^T\right)^{+}$ and its corresponding empirical estimation. Therefore, by denoting the difference between these two terms as 
\begin{align*}
    \Delta = (\hat{B}_i\hat{W}_i(\hat{B}_i\hat{W}_i)^T)^+ - \left(B^*W^*(B^*W^*)^T\right)^{+},
\end{align*} 
we can establish the connection between the true and empirical task relevance as
\begin{align}
\label{eq: 1}
    |\hat{\nu}_{i+1}(m)| - 2| \nu^*(m)|
    \lessapprox 2\left\lvert (B^* w_m^*)^T \Delta (B^*w_{M+1}^*) \right\lvert 
\end{align}

Now the minimization on source tasks shown in Eqn.~\eqref{eqn: ERM 1} ensures that 
\begin{align*}
    \|B^*W^* - \hat{B}_i\hat{W}_i\|_F \leq \sigma \text{poly}(d,M) \sqrt{\epsilon_i}.
\end{align*}
This helps us to further bound the $(\hat{B}_i\hat{W}_i(\hat{B}_i\hat{W}_i)^T) - \left(B^*W^*(B^*W^*)^T\right)$ term, which can be regarded as a perturbation on the underlying matrix $B^*W^*(B^*W^*)^T$.
Then by using the generalized inverse matrix theorem \cite{Kovanic1979}, we can show that the inverse of the perturbed matrix is close to its original matrix on some low dimension space. 

Therefore, we can upper bound Eqn.~\eqref{eq: 1} by $\sigma\sqrt{\epsilon_i}$.
We repeat the same procedure to lower bound the $\frac{1}{2}|\nu^*(m)| -|\hat{\nu}_{i+1}(m)|$. Combining these two, we have
\begin{align*}
    |\hat{\nu}_{i+1}(m)| \in \left[ \frac{1}{2}|\nu^*(m)| - \frac{7}{16}\sigma \sqrt{\epsilon_i}, 2|\nu^*(m)|  + 2\sigma\sqrt{\epsilon_i}\right]
\end{align*}

This directly lead to the result based on whether $\nu^*(m) \geq \sigma\sqrt{\epsilon}_i$ or not.
\end{proofsketch}

\textbf{Step 2:} Now we prove the following two main lemmas on the final accuracy and the total sample complexity.

Define event $\calE$ as the case that, for all epochs, the closeness between $\hat{\nu}_i$ and $\nu^*$ defined Lemma~\ref{lem: nu estimation error(main)} has been satisfied.

\begin{lemma}[Accuracy on each epoch (informal)]
\label{lem: main accuracy}
Under $\calE$, after the epoch $i$, we have $\textsc{ER}(\hat{B},\hat{w}_{M+1})$ roughly upper bounded by
\begin{align*}
    \frac{\sigma^2}{\beta} \left(KM + Kd + \log\frac{1}{\delta} \right)  s_i^* \epsilon_i^2
    + \frac{\sigma^2\left(K + \log(1/\delta) \right)}{n_{M+1}}
\end{align*}
where $ s_i^* = \min_{\gamma \in [0, 1]} (1- \gamma) \| \nu^* \|_{0,\gamma}^i + \gamma M$ \\
and $\| \nu \|_{0,\gamma}^i := |\{ m : \nu_m > \sqrt{ \gamma } \epsilon_i \}|$.

\end{lemma}

\begin{proofsketch}
As we showed in Section~\ref{sec: warm up}, the key for calculating the accuracy is to upper bound $\sum_m \frac{\nu^*(m)^2}{n_m^i}$. Similarly to Section~\ref{sec: warm up}, we employ the decomposition
\begin{align*}
 & \sum_m\!\tfrac{\nu^*(m)^2}{n_m^i}\!\left(\one\{ |\nu^*(m)|\!>\!\sqrt{\gamma}\epsilon_{i}\}
\!+\!\one\{ |\nu^*(m)|\!\leq\!\sqrt{\gamma}\epsilon_{i}\} \right).
\end{align*}
The last sparsity-related term can again be easily upper bounded by $\order((1-\| \nu^* \|_{0,\gamma}^i)\sigma^2\gamma\epsilon_i^2)$.

Then in order to make a connection between $n_m^i$ and $\nu^*(m)$ by using Lemma~\ref{lem: nu estimation error(main)}, we further decompose the first term as follows and get the upper bound
\begin{align*}
    &\sum_m \frac{\nu^*(m)^2}{n_m^i}\one\{ |\nu^*(m)| > \sigma\sqrt{\epsilon_{i-1}}\} \\
            & \quad + \sum_m\frac{\nu^*(m)^2}{n_m^i}\one\{ \sigma\sqrt{\gamma}\epsilon_{i}\leq |\nu^*(m)| \leq \sqrt{\epsilon_{i-1}}\}\\
    &\lessapprox \sum_m \frac{\hat{\nu}_{i}^2}{ n_m^i}\one\{ |\nu^*(m)| > \sigma\sqrt{\epsilon_{i-1}}\} \\
            & \quad + \sum_m  \frac{\sigma^2\epsilon_{i}}{n_m^i}\one\{ \sqrt{\gamma}\epsilon_{i-1}\leq |\nu^*(m)| \leq \sigma\sqrt{\epsilon_{i-1}}\}\\
    &\leq \sum_m \epsilon_i^2/\beta\one\{ |\nu^*(m)| > \sigma\sqrt{\epsilon_{i-1}}\} \\
            & \quad +  \sum_m \sigma^2\epsilon_i^2/\beta\one\{ \sqrt{\gamma}\epsilon_{i}\leq |\nu^*(m)| \leq \sigma\sqrt{\epsilon_{i-1}}\}\\
    &\leq  \order( \| \nu^* \|_{0,\gamma}^i\epsilon_i^2/\beta)
\end{align*}
where the second inequality is from the definition of $n_m^i$.
\end{proofsketch}

\begin{lemma}[Sample complexity on each epoch(informal)]
\label{lem: main complexity}
Under $\calE$, after the epoch $i$, We have the total number of training samples from source tasks upper bounded by 
\vspace{-5pt}
\begin{align*}
    \order\left(\beta (M\varepsilon^{-1} +  \|\nu^*\|_2^2\varepsilon^{-2})\right) + \text{low-order term} \times \Gamma.
\end{align*}
\end{lemma}
\vspace{-10pt}
\begin{proofsketch}
For any fixed epoch $i$, by definition of $n_m^i$, we again decompose the summed source tasks based on $\nu^*(m)$ and get the total sample complexity as follows
\begin{align*}
    &\sum_{m+1}^M \beta\hat{\nu}_i^2(m)\epsilon_i^{-2} \one\{ |\nu^*(m)| > \sigma \sqrt{\epsilon_{i-1}} \} \\
        & \quad + \sum_{m+1}^M \beta\hat{\nu}_i^2(m)\epsilon_i^{-2} \one\{ |\nu^*(m)| \leq \sigma \sqrt{\epsilon_{i-1}}\}
        + M \beta\epsilon_i^{-1}
\end{align*}
Again by replacing the value of $\hat{\nu}$ from Lemma~\ref{lem: nu estimation error(main)}, we can upper bounded second term in terms of $\nu^*$ and we can also show that the third term is low order $\epsilon^{-1}$. 
\end{proofsketch}

Theorem~\ref{thm: main} follows by combining the two lemmas.


\begin{figure*}[h]
    \centering
    \includegraphics[scale=0.38]{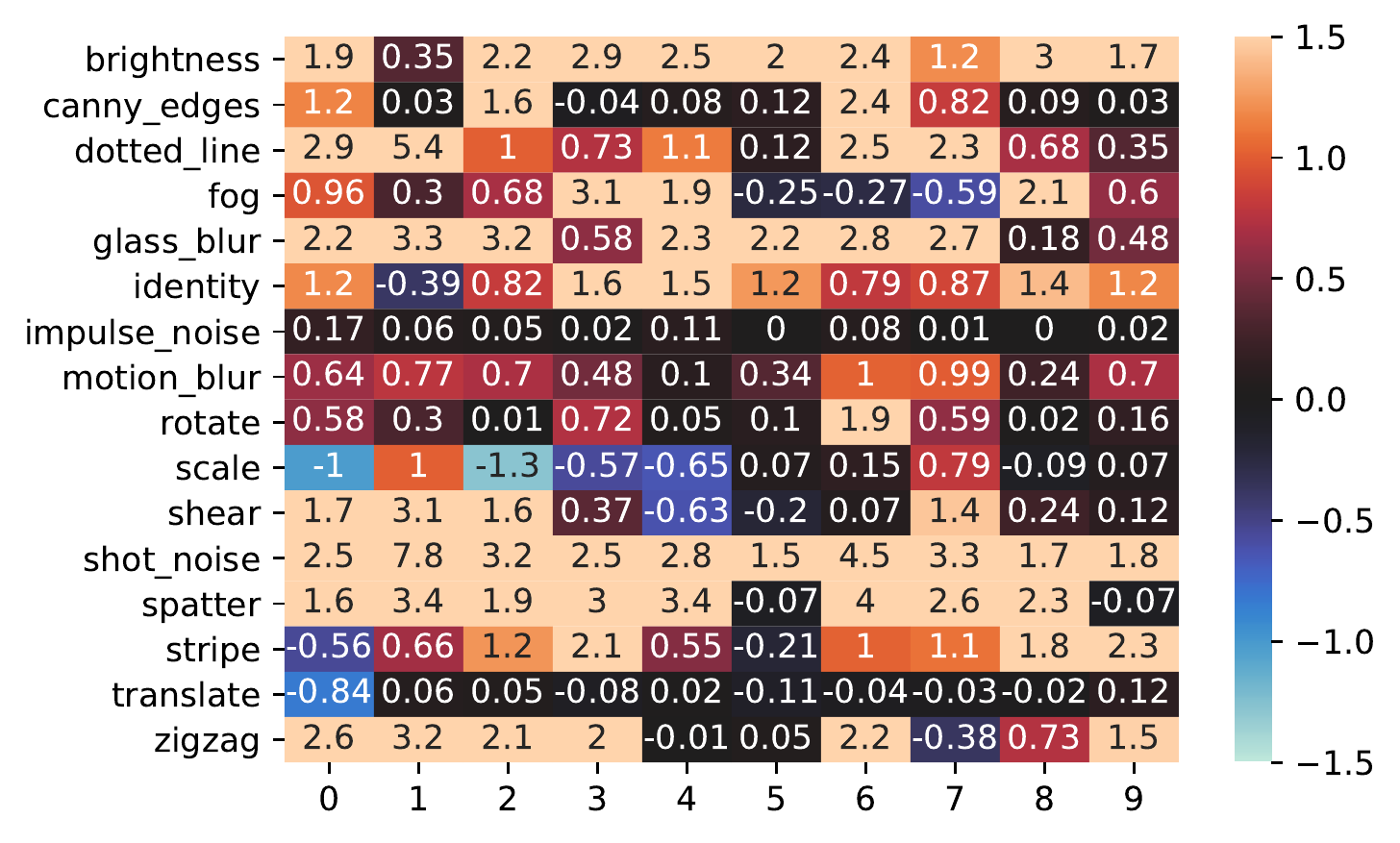}
    \includegraphics[scale=0.2, width=0.25\linewidth]{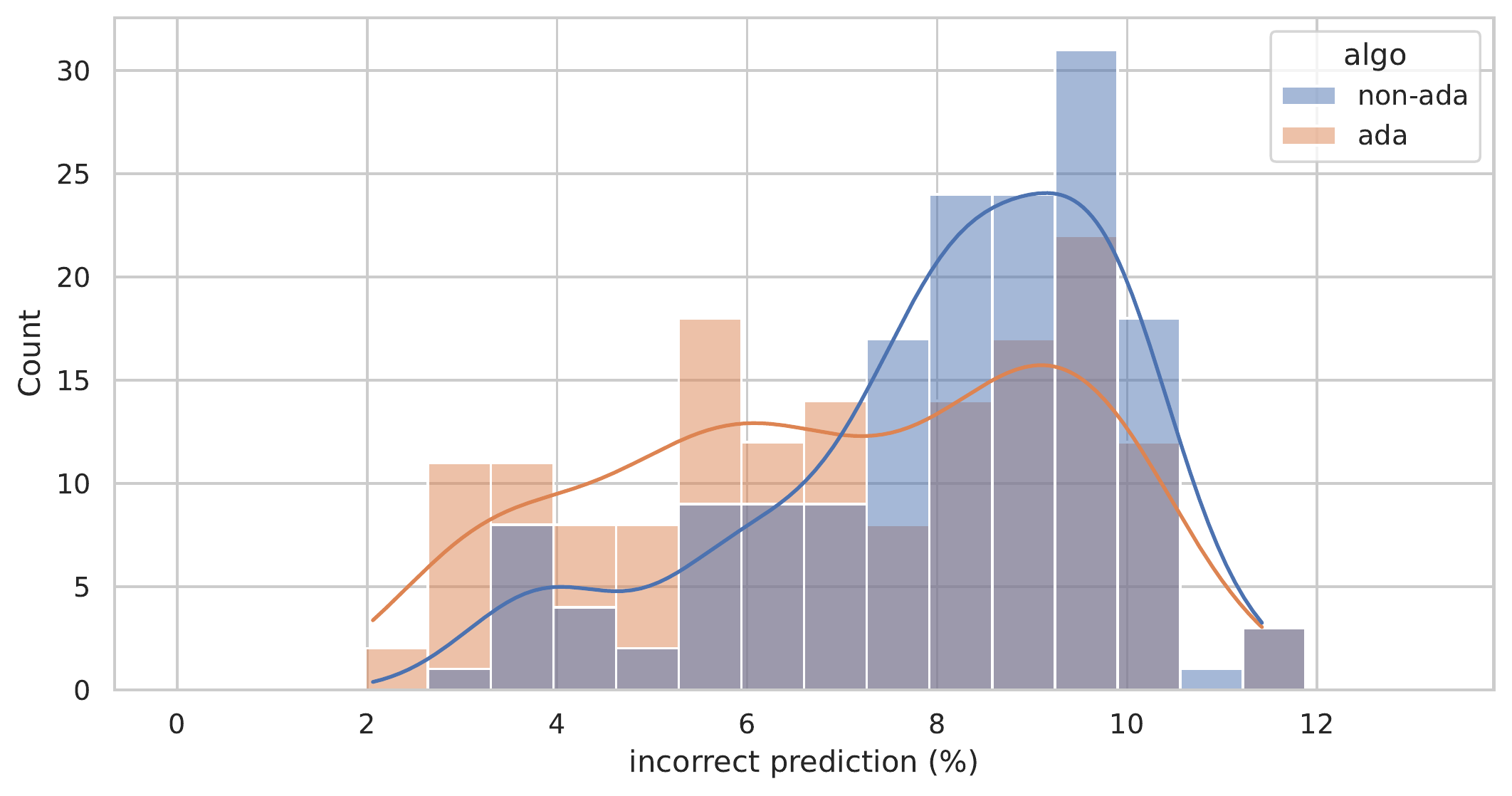}
    \includegraphics[scale=0.3]{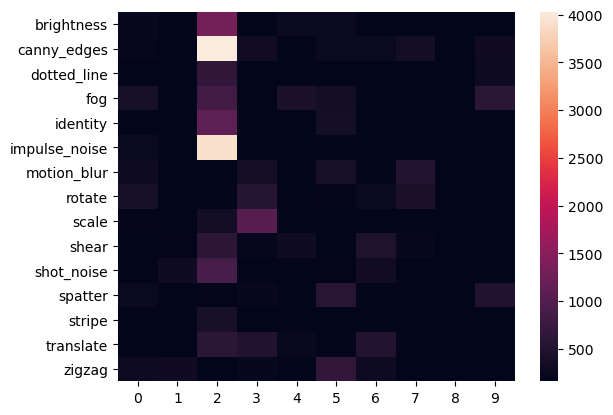}
    \vspace{-11px}
    \caption{
        \label{fig: linear summary}
    \textbf{Performance between the adaptive (ada) and the non-adaptive (non-ada) algorithm on linear representation.}
    {\small
    \textbf{Left}: The prediction difference (in \%)  between ada and non-ada for all target tasks. The larger is the better. Respectively y-axis denotes noise type and x-axis denotes binarized label, with each grid representing a target task, \textit{e.g.}, the grid at the top left corner stands for target task \textit{brightness\_0}.  In summary, the adaptive algorithm achieves $1.1 \%$ higher average accuracy than the non-adaptive one and results same or better accuracy in 136 out of 160 tasks. \textbf{Middle}: Histogram summary of incorrect prediction (left is better). There is a clear shift for adaptive algorithm towards left.  \textbf{Right:} Sampling distribution for the target task \textit{glass\_blur\_2}. Respectively, the plot shows numbers of samples from each source tasks at the beginning of epoch 3 by running adaptive algorithm. The samples clearly concentrated on several \textit{X\_2} source tasks, which meets our intuition that all "2 vs. others" tasks should has closer connection with the \textit{glass\_blur\_2} target task.}}
    \vspace{-10px}
\end{figure*}

\begin{figure*}[h]
    \centering
    \includegraphics[scale=0.38]{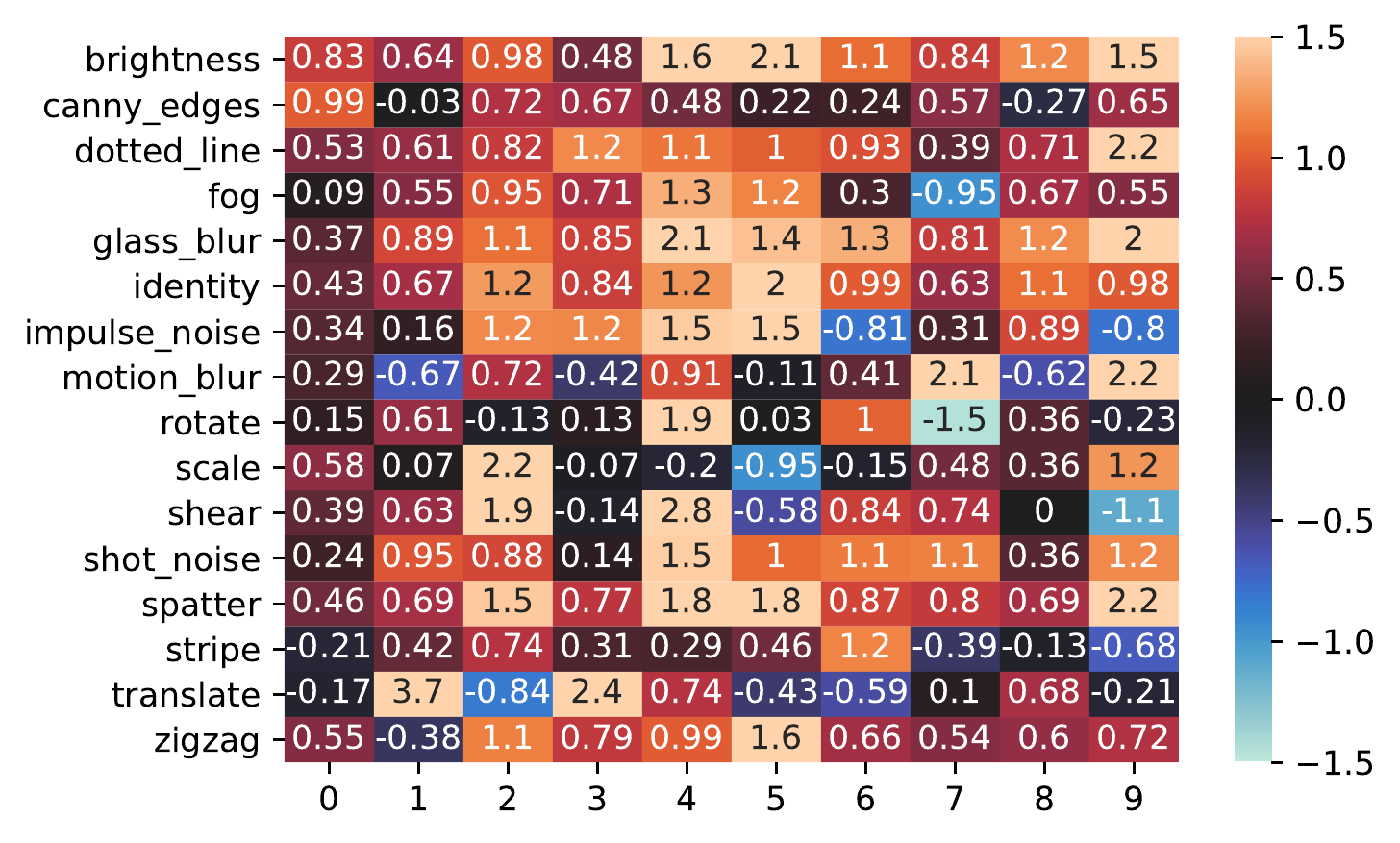}
    \includegraphics[scale=0.2]{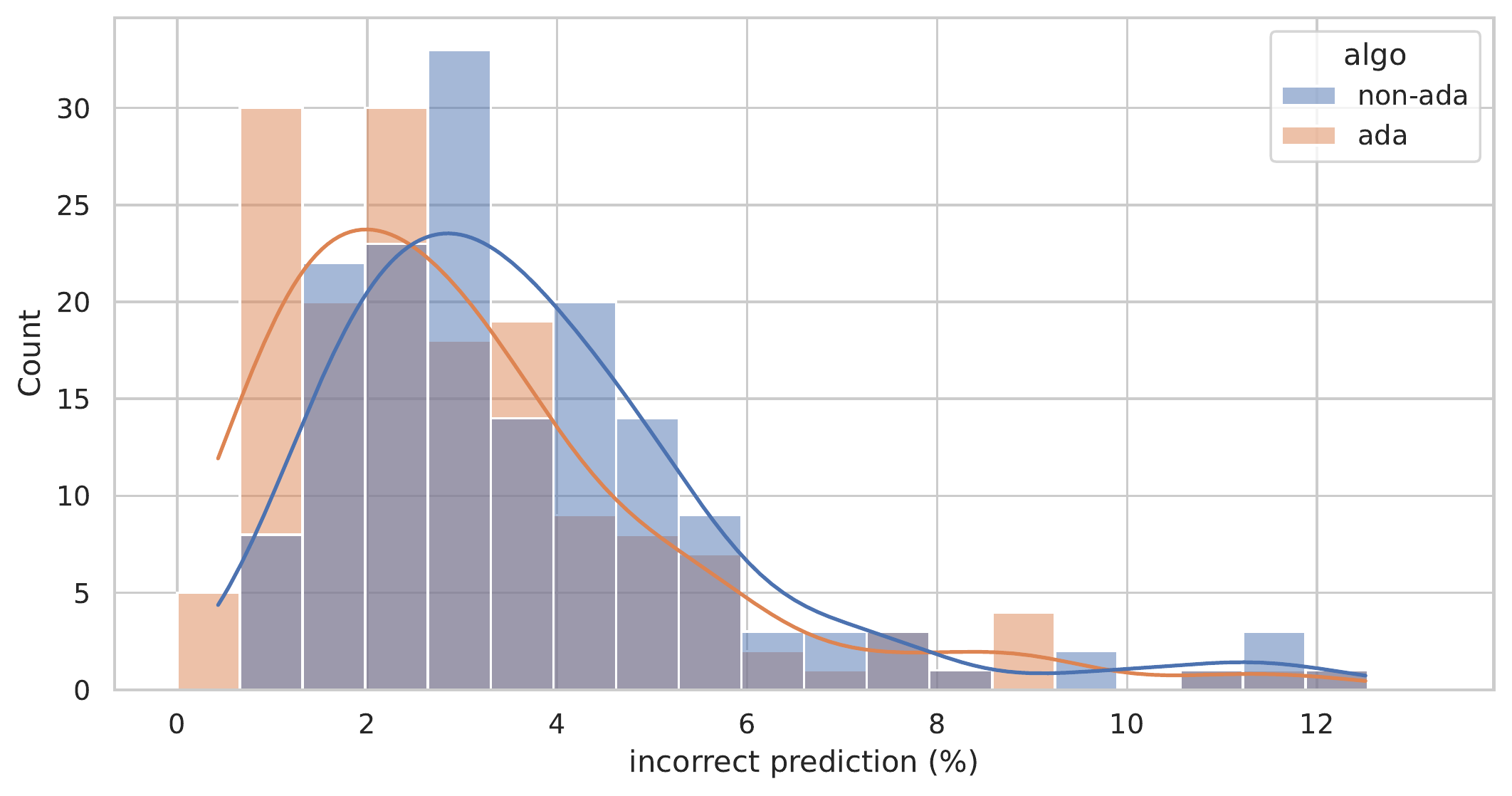}
    \includegraphics[scale=0.3]{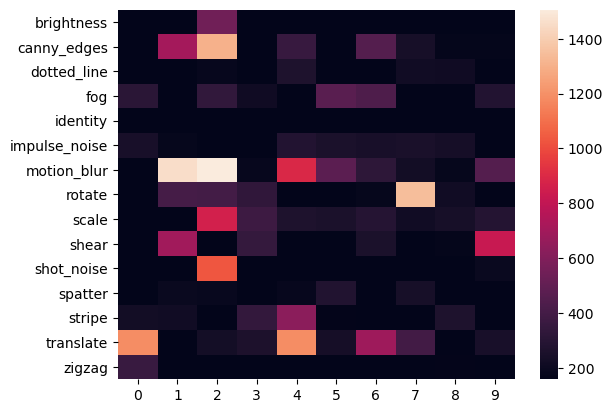}
    \vspace{-11px}
    \caption{
        \label{fig: cnn_summary}
    \textbf{Performance between the adaptive (ada) and the non-adaptive (non-ada) algorithm on Convnet.}
    {\small 
    \textbf{Left}:  The prediction difference (\%)  between ada and non-ada for all target tasks. (See more explanations for notations in Figure~\ref{fig: linear summary}.)
    In summary, the adaptive algorithm achieves $0.68\%$ higher average accuracy than the non-adaptive one and results same or better accuracy in 133 out of 160 tasks. 
    \textbf{Middle}: Histogram summary of incorrect prediction (left is better). There is a clear for adaptive algorithm towards left. Although the average performance improvement is smaller than in the linear representation, the relative improvement is also significant given the already good baseline performance (most prediction error are below $6\%$ while in linear most are above $6\%$). \textbf{Right: }Sample distribution for target task as \textit{glass\_blur\_2}. A large portion of samples again concentrate on several \textit{X\_2} source tasks, which meets our intuition that all "2 vs. others" tasks should has closer connection with \textit{glass\_blur\_2} target task. But the overall sample distribution is more spread compared to the one on linear representation. 
    }
    }
    \vspace{-10px}
\end{figure*}
\vspace{-5pt}

\section{Experiments}
\label{sec: experiment}

In this section, we empirically evaluate our active learning algorithm for multi-task by deriving tasks from the corrupted MNIST dataset (MNIST-C) proposed in \citet{mu2019mnist}.
While our theoretical results only hold for the linear representations, our experiments demonstrate the effectiveness our algorithm on neural network representations as well.
We show that our proposed algorithm:
(1) achieves better performance when using the same amount of source samples as the non-adaptive sampling algorithm, and
(2) gradually draws more samples on  important source tasks.  

\subsection{Experiment setup}



\paragraph{Dataset and problem setting.}
The MNIST-C dataset is a comprehensive suite of 16 different types of corruptions applied to the MNIST
test set. To create source and target tasks, we divide each sub-dataset with a specific corruption into 10 tasks by applying one-hot encoding to $0-9$ labels. Therefore, we have $160$ tasks in total, which we denote as "corruption type + label". For example, \textit{brightness\_0} denotes the data corrupted by brightness noise and are relabeled to $1/0$ based on whether the data is number 0 or not.  We choose a small number of fixed samples from the target task to mimic the scarcity of target task data. On the other hand, we set no budget limitation on source tasks. We compare the performance of our algorithm to the non-adaptive uniform sampling algorithm, where each is given the same number of source samples and same target task dataset. 

\vspace{-10pt}

\paragraph{Models.}
We start with the linear representation as defined in our theorem and set $B \in \fR^{28*28 \times 50}$ and $w_m^i \in \fR^{50}$. Note that although the MNIST problem is usually a classification problem with cross-entropy loss, here we model it as a regression problem with $\ell_2$ loss to align with the setting studied in this paper. Moreover, we also test our algorithm with 2-layer ReLU \textit{convolutional  neural nets} (CNN) followed by fully-connected linear layers, where all the source tasks share the same model except the last linear layer, also denoted as $w_m^i \in \fR^{50}$. 

\vspace{-10pt}




\paragraph{AL algorithm implementation.}
We run our algorithm iteratively for 4 epochs. The non-adaptive uniform sampling algorithm is provided with the same amount of source samples.
There are some difference between our proposed algorithm and what is implemented here. 
First we re-scale some parameters from the theorem to account for potential looseness in our analysis. 
Moreover, instead of drawing fresh i.i.d samples for each epoch and discarding the past, in practice, we reuse the samples from previous epochs and only draw what is necessary to meet the required number of samples of the current epoch. 
This introduces some randomness in the total source sample usage. 
For example, we may only require $100$ samples from the source task A for the current epoch, but we may have sampled 200 from source task A in the previous epoch. So we always sample equal or less than non-adaptive algorithm in a single epoch.
Therefore, in our result shown below, the total source sample numbers varies across target tasks. But we argue that this variation are roughly at the same level and will not effect our conclusion. Please refer to Appendix~\ref{subesc: implement details} for details.

\vspace{-7pt}
\subsection{Results}

\paragraph{Linear representation.} We choose 500 target samples from each target task. After 4 epochs, we use in total around 30000 to 40000 source samples. As a result, our adaptive algorithm frequently outperforms the non-adaptive one as shown in Figure~\ref{fig: linear summary}. 

For those cases where gains are not observed, we conjecture that those tasks violate our realizable assumptions more than the others. We provide a more detailed discussion and supporting results for those failure cases in Appendix~\ref{subsec: bad examples on linear (app)}.
Next, we investigate the sample distribution at the beginning epoch 3. We show the result for \textit{glass\_blur\_2} as a representative case with more examples in the Appendix~\ref{subsec: good example linear}. From the figure, we can clearly see the samples concentrate on more target-related source tasks.

\vspace{-10pt}

\paragraph{Convnet.} We choose 200 target samples from each target task. After 4 epochs, we use in total around 30000 to 40000 source samples. As a result, our adaptive algorithm again frequently outperforms the non-adaptive one as shown in Figure~\ref{fig: linear summary}.
Next, we again investigate the sample distribution at the beginning epoch 3 and show a representative result (more examples in Appendix~\ref{subsec: good example cnn}). First of all, there are still a number  of source samples again concentrating on "2 vs. others". 
The reader may notice some other source tasks also contribute a relatively large amount of sample complexity. This is actually a typical phenomenon in our experiment on convnets, which is seldom observed in the linear representation. This might be due to the more expressive power of CNN that captures some non-intuitive relationships between some source tasks and the target task. Or it might be simply due to the estimation error since our algorithm is theoretically justified only for the realizable linear representation. We provide more discussion in Appendix~\ref{subsec: bad examples on cnn (app)}.

\section{Conclusion and future work}
Our paper takes an important initial step to bring techniques from active learning to representation learning. There are many future directions. From the theoretical perspective, it is natural to analyze some fine-tuned models or even more general models like neural nets as we mentioned in the related work section. From the empirical perspective, our next step is to modify and apply the algorithm on more complicated CV or NLP datasets and further analyze its performance.


\bibliography{ref}
\bibliographystyle{icml2022}

\newpage

\appendix
\onecolumn


\section{Appendix structure}
In Appendix~\ref{sec: Notation(app)}, we define the commonly used notations in the following analysis. In Appendix~\ref{sec: modified claims (app)}, we define three high probability events and prove three claims as variants of original results in \citet{du2020few}. All these events and claims are widely used in the following theoretical analysis. Then we give formal proofs of Theorem~\ref{thm: warm-up} and Theorem~\ref{thm: compare to uniform} in Appendix~\ref{sec: warmup(app)} and formal proofs of Theorem~\ref{thm: main} in Appendix~\ref{sec: main analysis(app)}. Finally, we show more comprehensive results of experiment in Appendix~\ref{sec: experiment(app)}. 

\section{Notation}
\label{sec: Notation(app)}
\begin{itemize}
    \item Define $\hDelta = \hB\hW - B^*W^*$ and correspondingly define $\hDelta^i = \hB_i\hW_i - B^*W^*$ if the algorithm is divided into epochs.
    \item Define $\hDelta_m = \hB\hw_m - B^*w_m^*$ and correspondingly define $\hDelta_m^i = \hB_i\hw_m^i - B^*w_m^*$ .
    \item Restate $P_{A}^\perp=I - A\left(A^{\top} A\right)^{\dagger} A^{\top}$
    \item Restate that $\hSigma_m = (X_m)^\top X_m$ and correspondingly $\hSigma_m^i = (X_m^i)^\top X_m^i$ if the algorithm is divided into epochs.
    \item Restate that $\Sigma_m = \E\left[\hSigma_m\right]$ and correspondingly $\Sigma_m^i =  \E\left[\hSigma_m^i\right]$.
    \item Define $\kappa = \frac{\lambda_{max}(\Sigma)}{\lambda_{min}(\Sigma)}$, recall we assume all $\Sigma_m = \Sigma$. Note that in the analysis for adaptive algorithm, we assume identity covariance so $\kappa = 1$.
    \item For convenience, we write $\sum_{m=1}^M$ as $\sum_m$.
    \item If the algorithm is divided into epochs, we denote the total number of epoch as $\Gamma$.
\end{itemize}

\section{Commonly used claims and definitions}

\subsection{Co-variance concentration guarantees}
We define the following guarantees on the feature covariance concentration that has been used in all proofs below.
\begin{align*}
    & \calE_\text{source} = \{ 0.9 \Sigma_m \leq \hSigma_m \leq 1.1 \Sigma_m, \forall m \in [M] \} \\
    & \calE_\text{target1} = \{ 0.9  B_1^\top B_2 \leq B_1^\top \hSigma_{M+1} B_2 \leq 1.1  B_1^\top B_2, \text{for any orthonormal} B_1,B_2 \in  \fR^{d \times K} \mid \Sigma_{M+1} = I\}\\
    & \calE_\text{target2} = \{ 0.9 \Sigma_{M+1} \leq B^\top \hSigma_{M+1} B \leq 1.1 \Sigma_{M+1}, \text{for any } B \in  \fR^{d \times K}\}
\end{align*}
By Claim A.1 in \citet{du2020few}, we know that, as long as 
$n_m \gg d + \log(M/\delta), \forall m \in [M+1]$
\begin{align*}
    &\prob\left(\calE_\text{source}^i \right) \geq 1-\frac{\delta}{10},
\end{align*}
Moreover, as long as $n_m \gg K + \log(1/\delta)$, 
\begin{align*}
 & \prob(\calE_\text{target1}) \geq 1-\frac{\delta}{10}
 & \prob(\calE_\text{target2}) \geq 1-\frac{\delta}{10}.
\end{align*}

Correspondingly, if the algorithm is divided into epochs where each epoch we draw new set of data, then we define
\begin{align*}
    \calE_\text{source}^i = \{ 0.9 \Sigma_m \leq \hSigma_m^i \leq 1.1 \Sigma_m, \forall m \in [M] \} 
\end{align*}
Again, as long as $n_m^i \gg d + \log(M\Gamma/\delta), \forall m \in [M], \forall i \in [\Gamma]$, we have
\begin{align*}
    \prob\left(\bigcup_{i \in [\Gamma]}\calE_\text{source}^i \right) \geq 1-\frac{\delta}{5}
\end{align*}

Notice $\calE_\text{target1}$ will only be used in analyze the main active learning Algorithm~\ref{alg:main_1} while the other two is used in both Algorithm~\ref{alg:known} and \ref{alg:main_1}.

\subsection{Claims guarantees for unequal sample numbers from source tasks}
\label{sec: modified claims (app)}

Here we restate two claims and one result (not written in claim) from \cite{du2020few}, and prove that they still hold when the number of samples drawn from each of the source tasks are not equal, as long as the general low-dimension linear representation is satisfied as defined in Definition~\ref{assump: linear model}.  No benign setting like Definition~\ref{assump: identity covar} is required.

\begin{algorithm}[h] 
\caption{General sample procedure}
\begin{algorithmic}[1] 
\STATE For each task $m$, draw $n_m$ i.i.d samples from the corresponding offline dataset denoted as $\{X_m,Y_m\}_{m=1}^M$
\STATE Estimate the models as
\begin{align*}
&\hat{\phi}, \hat{W} = \argmin_{\phi \in \Phi,\hat{W}=[\hat{w}_1,\hat{w}_2,\ldots]} \sum_{m=1}^M \| \phi(X_m) \hat{w}_m - Y_m\|^2\\
&\hat{w}_{M+1} = \argmin_{w} \|\hat{\phi}(X_{M+1}^i) w - Y_{M+1}\|^2 
\end{align*}
\end{algorithmic}
\end{algorithm} 

Specidically, consider the above procedure, we show that the following holds for any $\{n_m\}_{m=1}^M$.

\begin{claim}[Modified version of Claim A.3 in \cite{du2020few}] 
\label{claim 1}
Given $\calE_\text{source}$, with probability at least $1-\delta/10$,
\begin{align*}
    \sum_m \|X_m \hDelta_m \|_2^2
    \leq \sigma^2 \left(KM + Kd\log( \kappa (\sum_m n_m)/M) + \log(1/\delta) \right)
\end{align*}
\end{claim}
\begin{proof}
We follow nearly the same steps as the proof in \cite{du2020few}, so some details are skipped and we only focus on the main steps that require modification. 
Also we directly borrow some notations including $\overline{V}, \calN,r$ from the original proof and will restate some of them here for clarity. 

\paragraph{Notation restatement} Since $\operatorname{rank}(\hDelta) \leq 2 K$, we can write $\Delta=V R=\left[V \boldsymbol{r}_{1}, \cdots, V \boldsymbol{r}_{M}\right]$ where $V \in \mathcal{O}_{d, 2 K}$ and $R=\left[\boldsymbol{r}_{1}, \cdots, \boldsymbol{r}_{M}\right] \in \mathbb{R}^{2 K \times M} .$ Here $\mathcal{O}_{d_{1}, d_{2}}\left(d_{1} \geq d_{2}\right)$ is the set of orthonormal $d_{1} \times d_{2}$ matrices (i.e., the columns are orthonormal). For each $m \in[M]$ we further write $X_{m} V=U_{m} Q_{m}$ where $U_{m} \in \mathcal{O}_{n_{1}, 2 K}$ and $Q_{m} \in \mathbb{R}^{2 K \times 2 K}$. To cover all possible $V$, we use an $\epsilon$-net argument, that is, there exists any fixed $\overline{V} \in \mathcal{O}_{d, 2 K}$ and  there exists an $\epsilon$-net $\mathcal{N}_\epsilon$ of $\mathcal{O}_{d, 2 K}$ in Frobenius norm such that $\mathcal{N} \subset \mathcal{O}_{d, 2 K}$ and $|\mathcal{N}_\epsilon| \leq\left(\frac{6 \sqrt{2 K}}{\epsilon}\right)^{2 K d}$.  (Please refer to original proof for why such $\epsilon$-net exists)

Now we briefly state the proofs.

\textbf{Step 1:} 
\begin{align*}
    \sum_m \|X_m(\hat{B}\hat{w}_m - B^*w_m^*) \|_2^2
    \leq \sum_{m=1}^M \langle Z_m, X_m\overline{V}_m r_m \rangle
        + \sum_{m=1}^M \langle Z_m, X_m(V-\overline{V}_m) r_m \rangle
\end{align*}
This comes from the first three lines of step 4 in original proof.
For the first term, with probability $1-\delta/10$ , by using standard tail bound for $\chi^{2}$ random variables and the $\epsilon$-net argument (details in eqn.(28) in original proof), we have it upper bounded by 
\begin{align*}
 \sigma \sqrt{KM + \log(|\calN_\epsilon|/\delta)}\sqrt{\sum_m^M \|X_m Vr_m\|_2^2} + \sigma \sqrt{KM + \log(|\calN_\epsilon|/\delta)}\sqrt{ \sum_m^M \|X_m (\overline{V}-V) r_m\|_2^2}
\end{align*}

And for the second term, since $\sigma^{-2}\sum_{m}\|Z_m\|^{2} \sim \chi^{2}\left(\sum_{m=1}^M n_m \right)$, again by using standard tail bound (details in eqn.(29) in original proof), we have that with high probability $1-\delta/20$, it is upper bounded by 
\begin{align*}
    \sum_{m=1}^M \langle Z_m, X_m(V-\overline{V}_m) r_m \rangle
    \lessapprox \sigma \sqrt{\sum_{m=1}^M n_m + \log(1/\delta)}
      \sqrt{\sum_m \| X_m(V-\overline{V}_m) r_m\|_2^2} 
\end{align*}

\textbf{Step 2: } Now we can further bound the term $\sqrt{\sum_m \| X_m(V-\overline{V}_m) r_m\|_2^2}$ by showing
\begin{align*}
    \sum_m \| X_m(V-\overline{V}_m) r_m\|_2^2
    & \leq \sum_m \|X_m\|_F^2 \|V-\overline{V}\|_F^2 \|r_m\|_2^2 \\
    & \leq 1.1 \overline{\lambda} \sum_m n_m \|V-\overline{V}\|_F^2 \|r_m\|_2^2 \\
    &\leq 1.1 \overline{\lambda}\epsilon^2 \sum_m n_m \|r_m\|_2^2 \\
    & \leq 1.1 \overline{\lambda} \epsilon^2 \sum_m n_m \|\hDelta_m\|_2^2 \\
    & \leq  1.1 \kappa \epsilon^2 \sum_m \|X_m (\hDelta_m)\|_2^2
\end{align*}
Note this proof is the combination of step 2 and step 3 in the original proof. The only difference here is $n_m$ is different for each $m$ so you need to be more careful on those upper and lower bounds.

\textbf{Step 3:} Finally, we again use the self-bounding techniques. Recall that we have
\begin{align*}
    \sqrt{\sum_m \|X_m (\hDelta_m)\|_2^2}
    & \leq \sqrt{\sum_m^M \|Z_m\|_2^2} \sqrt{\sum_m^M \|X_m (\hDelta_m)\|_2^2}
\end{align*}
By rearranging the inequality and the distribution of $Z_m$, we have
\begin{align*}
    \sum_m \|X_m (\hDelta_m)\|_2^2 \leq  \sigma^2 \left(\sum_{m=1}^M n_m + \log(1/\delta)\right)
\end{align*}

\textbf{Step 4: } Now replace these into the inequality in step 1, we have
\begin{align*}
    \sum_m \|X_m\hDelta_m \|_2^2
    \leq \sigma \sqrt{KM + \log(|\calN_\epsilon|/\delta)}\sqrt{\sum_m \|X_m\hDelta_m \|_2^2} + 2\epsilon\sigma^2\left( \sum_m n_m + \log(1/\delta)\right)
\end{align*}
Then rearrange the inequality and choose proper $\epsilon$, we get the result.

\end{proof}

\begin{claim}[Modified version of Claim A.4 in \cite{du2020few}]
\label{claim 2}
Given $\calE_\text{source}$ and $\calE_\text{target2}$, with probability at least $1-\delta/10$,
\begin{align*}
    \|P_{X_{M+1}\hat{B}}^{\bot}X_{M+1}B^* \widetilde{W}^*\|_F^2
    \leq 1.3 n_{M+1}\sigma^2 \left(KT + Kd\log( (\kappa \sum_m n_m)/M) + \log\frac{1}{\delta} \right)
\end{align*}
where $\widetilde{W}^* = W^* \sqrt{\text{diag}([n_1,n_2, \ldots,n_M])}$
\end{claim}
\begin{proof}
    The proof is almost the same as the first part of the proof except we don't need to extract $n_m$ out.
    \begin{align*}
        \sum_{m} \|X_m(\hat{B}\hw_m - B^*w_m^*) \|_2^2
        &\geq \sum_{m} \|P_{X_m\hat{B}}^{\bot}X_m B^*w_m^*\|^2 \\
        &\geq 0.9 \sum_{m} n_m \|P_{\Sigma_{m}\hat{B}}^{\bot}\Sigma_{m}B^*w_m^*\|^2 \\
        & = 0.9 \sum_{m} n_m \|P_{\Sigma_{M+1}\hat{B}}^{\bot}\Sigma_{M+1}B^*w_m^*\|^2\\
        &= 0.9 \|P_{\Sigma_{M+1}\hat{B}}^{\bot}\Sigma_{M+1}B^* \Tilde{W}^*\|^2\\
        &\geq \frac{0.9}{1.1} \frac{1}{n_{M+1}}\|P_{X_{M+1}\hat{B}}^{\bot}X_{M+1}B^* \widetilde{W}^*\|_F^2
    \end{align*}
where the first and second inequality is the same as original proof.  The third equation comes from our assumption that all $\Sigma_m$ are the same and the forth equation is just another form of the above term. The last inequality comes from the same reason as second equality, which can again be found in original proof. 

Now by using Claim~\ref{claim 1} as an upper bound, we get our desired result.

Basically, this claim is just another way to write Claim A.4 in \cite{du2020few}. Here we combined the $n_m$ with $W^*$ and in original proof they extract $n_m$ can lower bound $W^*$ with its minimum singular value since in their case all $n_m$ are the same. 
\end{proof}

\begin{claim}
\label{claim: upper bound of er}
Given $\calE_\text{source}$ and $\calE_\text{target2}$, with probability at least $1-\delta/5$,
\begin{align*}
\operatorname{ER}\left(\hat{B}, \hat{\boldsymbol{w}}_{M+1}\right) 
&\leq\frac{1}{n_{M+1}}\left\|P_{X_{M+1} \hat{B}}^{\perp} X_{M+1} B^{*} \boldsymbol{w}_{M+1}^{*}\right\|_{F}^{2}+\sigma^2\frac{K + \log(1/\delta)}{n_{M+1}} .
\end{align*}
\end{claim}
\begin{proof}
    This bound comes exactly from some part of \textit{Proof of Theorem 4.1}. Nothing need to change.
\end{proof}
\section{Analysis for Warm-up}
\label{sec: warmup(app)}

\subsection{Proof for Theorem~\ref{thm: warm-up}}
Suppose event $\calE_\text{source}$ and $\calE_\text{target2}$ holds, then we have with probability at least $1-\frac{3\delta}{10}$,
\begin{align*}
    \textsc{ER}(\hat{B},\hat{w}_{M+1})
    &\leq \frac{\|P_{X_{M+1}\hat{B}}^{\bot}X_{M+1}B^* w_{M+1}^*\|^2}{n_{M+1}} +\sigma^2\frac{K + \log(1/\delta)}{n_{M+1}}\\
    & = \frac{1}{n_{M+1}}\|P_{X_{M+1}\hat{B}}^{\bot}X_{M+1}B^* \widetilde{W}^*\tilde{\nu}^*\|_2^2+\sigma^2\frac{K + \log(1/\delta)}{n_{M+1}}\\
    &\leq \frac{1}{n_{M+1}} \|P_{X_{M+1}\hat{B}}^{\bot}X_{M+1}B^* \widetilde{W}^*\|_F^2\|\tilde{\nu}^*\|_2^2+\sigma^2\frac{K + \log(1/\delta)}{n_{M+1}}\\
    & = 1.3 \sigma^2 \left(KT + Kd\log( (\kappa \sum_m n_m)/M) + \log\frac{1}{\delta} \right) \|\tilde{\nu}^*\|_2^2+\sigma^2\frac{K + \log(1/\delta)}{n_{M+1}}
\end{align*}
where $\hat{\nu}^*(m) = \frac{\nu^*(m)}{\sqrt{n_m}}$. Here the first inequality comes from Claim~\ref{claim: upper bound of er} and the last inequality comes from Claim~\ref{claim 2}. By use both claim, we have probability at least $1-\frac{\delta}{10} - \frac{\delta}{5}$. The third inequality comes from holder's inequality.
 
The key step to our analysis is to decompose and upper bound $\|\tilde{\nu}^*\|_2^2$. Denote $\epsilon^{-2} = \frac{N_\text{total}}{\|\nu^*\|_2^2}$, we have, for any $\gamma \in [0,1]$,
\begin{align*}
    \sum_m \frac{\nu^*(m)^2}{n_m}\left(\one\{ |\nu^*(m)| > \sqrt{\gamma}\epsilon\} + \one\{ |\nu^*(m)| \leq \sqrt{\gamma}\epsilon\}\right)
    & \lessapprox \sum_m \left( \epsilon^{2}\one\{ |\nu^*(m)| > \sqrt{\gamma}\epsilon\} + \gamma\epsilon^2\one\{ |\nu^*(m)| \leq \sqrt{\gamma}\epsilon\}\right)\\
    & \leq \|\nu\|_{0,\gamma} \epsilon^2 + (M-\|\nu\|_{0,\gamma}) \gamma\epsilon^2\\
    & = (1-\gamma) \|\nu\|_{0,\gamma} \epsilon^2 + M\gamma\epsilon^2\\
    & \leq \frac{\|\nu^*\|_2^2}{N_\text{total}} \left((1-\gamma) \|\nu\|_{0,\gamma} + \gamma M \right)
\end{align*}
where the inequality comes from $n_m \geq \frac{1}{2}(\nu^*(m))^2\epsilon^{-2}$. 

Finally, combine this with the probability of  $\calE_\text{source}$ and $\calE_\text{target2}$, we finish the bound.

\subsection{Proof for Theorem~\ref{thm: compare to uniform}}

By the same procedure as before, we again get
\begin{align*}
    \textsc{ER}(\hat{B},\hat{w}_{M+1})
    \leq \sigma^2 \left(KT + Kd\log(\kappa (\sum_m n_m)/M) + \log\frac{1}{\delta} \right) \|\tilde{\nu}^*\|_2^2 +\sigma^2\frac{K + \log(1/\delta)}{n_{M+1}}
\end{align*}
Now due to uniform sampling, so we have all $n_m = N_\text{total}/M$, which means
\begin{align*}
    \|\tilde{\nu}^*\|_2^2 = \|\nu^*\|_2^2 \frac{M}{N_\text{total}}.
\end{align*}
Then we get the result by direction calculation.

\section{Analysis for Theorem~\ref{thm: main}}
\label{sec: main analysis(app)}

\subsection{Main analysis}


\textbf{Step 1: We first show that the estimated distribution over tasks $\hat{\nu}_i$ is close to the actual distribution $\nu^*$ for any fixed $i$. Notice that Assumption~\ref{assump: identity covar} is necessary for the proofs in this part.}

\begin{lemma}[Closeness between $\hat{\nu}_i$ and $\nu^*$]
\label{lem: nu estimation error(app)}
Under the Assumption~\ref{assump: identity covar}, given $\calE_\text{source}^i$ and $\calE_\text{target1}$, for any $i,m$, as long as $n_{M+1} \geq \frac{ 2000\epsilon_i^{-1}}{ \underline{\sigma}^4} $, we have with probability at least $1-\delta/10M\Gamma$
\begin{equation}
\label{eq: nu estimation error(app)}
    |\hat{\nu}_{i+1}(m)| \in
    \begin{cases}
       \left[|\nu^*(m)|/16, 4|\nu^*(m)|\right] & \text{if }  \nu^* \geq \sigma \sqrt{\epsilon_i}\\
       [0, 4\sqrt{\epsilon_i}] & \text{if } |\nu^*| \leq \sigma \sqrt{\epsilon_i}\\
    \end{cases}       
\end{equation}
We define this conditional event as 
\begin{align*}
    \calE_\text{relevance}^{i,m} = \left\{\text{Eqn.\eqref{eq: nu estimation error(app)} holds } \mid \calE_\text{source}^i, \calE_\text{target1} \right\}
\end{align*}

\end{lemma}

\begin{proof}
By the definition of $\nu^*$ and Lemma~\ref{lem: a closed form for estimated nu}, we have the following optimization problems,
\begin{align*}
    &\hat{\nu}_{i+1}
    =\argmin_\nu \|\nu\|_2^2  \\
    \text{s.t.} \quad
    & \sum_m  \alpha_m \left(\hat{\Sigma}_m^i B^*w_m^* 
        +\frac{1}{n_m^i} \left(X_m^i\right)^\top Z_m\right)\nu(m)
        =  \alpha_{M+1} \left(\hat{\Sigma}_M^i B^*w_{M+1}^*  +\frac{1}{n_{M+1}^i} \left(X_{M+1}^i\right)^\top Z_{M+1}\right)\\
    &\nu^* = \argmin_\nu \|\nu\|_2^2 \\
    \text{s.t.} \quad
    &  \sum_m w_m^*\nu(m) = w_{M+1}^*
\end{align*}
where $\alpha_m^i = \left( \hat{B}_i^\top \hat{\Sigma}_m^i\hat{B}_i \right)^{-1}\hat{B}_i^\top $.

Now we are ready to show that $\hat{\nu}_{i+1}$ is close to $\nu^*$ by comparing their closed form solution.

First by using the lemma~\ref{lem: KKT for full rank} based on the standard KKT condition, it is easy to get a closed form solution of $\nu^*$ , that is, for any $m$.
\begin{align*}
    \nu^*(m) &= (w_m^*)^\top \left(W^*(W^*)^\top \right)^{-1}w_{M+1}^* \\
    &= (B^* w_m^*)^\top \left(B^*W^*(B^*W^*)^\top \right)^{+} (B^*w_{M+1}^*) 
\end{align*}
where the last inequality comes from the fact that $B^*$ is orthonormal. And,
\begin{align*}
    \lvert \hat{\nu}_{i+1}(m) \rvert
    & = \left\lvert \left((\hat{\Sigma}_m^iB^*w_m^*)^\top  + \frac{1}{n_m^i}(Z_m^i)^\top X_m^i \right)\alpha_m^\top  (\hat{W}_i\hat{W}_i^\top )^\dagger  \alpha_{M+1} \left(\hat{\Sigma}_{M+1}^i B^*w_{M+1}^*+ \frac{1}{n_{M+1}^i}(X_{M+1}^i)^\top Z_{M+1}^i \right) \right\lvert\\
    &\leq 1.7 \left\lvert (B^*w_m^*)^\top \hat{B}_i(\hat{B}_i^\top \hat{B}_i)^{-1} (\hat{W}_i\hat{W}_i^\top )^\dagger  (\hat{B}_i^\top \hat{B}_i)^{-1}\hat{B}_i^\top  B^*w_{M+1}^* \right\lvert \\
        & \quad +1.3 \left\lvert \frac{1}{n_m^i}(Z_m^i)^\top X_m^i\hat{B}_i(\hat{B}_i^\top \hat{B}_i)^{-1} (\hat{W}_i\hat{W}_i^\top )^\dagger  (\hat{B}_i^\top \hat{B}_i)^{-1} \hat{B}_i^\top  B^*w_{M+1}^* \right\lvert  \\
        & \quad +1.3 \left\lvert (B^*w_m^*)^\top  \hat{B}_i(\hat{B}_i^\top \hat{B}_i)^{-1} (\hat{W}_i\hat{W}_i^\top )^\dagger  (\hat{B}_i^\top \hat{B}_i)^{-1} \hat{B}_i^\top  \frac{1}{n_{M+1}^i}(X_{M+1}^i)^\top Z_{M+1}^i \right\lvert  \\
        &\quad +\left\lvert \frac{1}{n_m^i}(Z_m^i)^\top X_m^i\hat{B}_i(\hat{B}_i^\top \hat{B}_i)^{-1} (\hat{W}_i\hat{W}_i^\top )^\dagger  (\hat{B}_i^\top \hat{B}_i)^{-1} \hat{B}_i^\top  \frac{1}{n_{M+1}^i}(X_{M+1}^i)^\top Z_{M+1}^i \right\lvert  \\
    &\leq 1.7 \left\lvert (B^*w_m^*)^\top \hat{B}_i (\hat{W}_i\hat{W}_i^\top )^\dagger  \hat{B}_i^\top  B^*w_{M+1}^* \right\lvert \\
        & \quad +1.3 \left\lvert \frac{1}{n_m^i}(Z_m^i)^\top X_m^i\hat{B}_i (\hat{W}_i\hat{W}_i^\top )^\dagger   \hat{B}_i^\top  B^*w_{M+1}^* \right\lvert \\
        & \quad +1.3 \left\lvert (B^*w_m^*)^\top  \hat{B}_i (\hat{W}_i\hat{W}_i^\top )^\dagger   \hat{B}_i^\top  \frac{1}{n_{M+1}^i}(X_{M+1}^i)^\top Z_{M+1}^i \right\lvert \\
        &\quad +\left\lvert \frac{1}{n_m^i}(Z_m^i)^\top X_m^i\hat{B}_i (\hat{W}_i\hat{W}_i^\top )^\dagger   \hat{B}_i^\top  \frac{1}{n_{M+1}^i}(X_{M+1}^i)^\top Z_{M+1}^i \right\lvert \\
    & \leq 2\left\lvert(B^*w_m^*)^\top  (\hat{B}_i\hat{W}_i(\hat{B}_i\hat{W}_i)^\top )^\dagger   B^*w_{M+1}^* \right\lvert  + \textit{noise term}(m)
\end{align*}
The first inequality comes from the definition of $\alpha_m$ and the event $\calE_\text{source},\calE_\text{target1}$. The second equality comes from that $\hat{B}_i$ are always orthonormal matrix. Notice that in practice this is not required. Finally we set the last three terms in the second inequality as \textit{noise term}(m), which is a low order term that we will shown later.

\textbf{----Sub-step 1 (Analyze the non noise-term): } We have the difference between $|\hat{\nu}_{i+1}(m)|$ and $2| \nu^*(m)|$ is
\begin{align*}
    &|\hat{\nu}_{i+1}(m)| - 2| \nu^*(m)| - \textit{noise term}(m)\\
    &\leq  2 \left\lvert(B^*w_m^*)^\top  (\hat{B}_i\hat{W}_i(\hat{B}_i\hat{W}_i)^\top )^\dagger   B^*w_{M+1}^* \right\lvert
        - 2 \left\lvert (B^* w_m^*)^\top \left(B^*W^*(B^*W^*)^\top \right)^{+} (B^*w_{M+1}^*) \right\lvert\\
    & \leq 2 \left\lvert (B^* w_m^*)^\top  \left((\hat{B}_i\hat{W}_i(\hat{B}_i\hat{W}_i)^\top )^\dagger  - \left(B^*W^*(B^*W^*)^\top \right)^{+}\right) (B^*w_{M+1}^*) \right\lvert \\
    & \leq 2  \left\lvert (B^* w_m^*)^\top  \left(B^*W^*(B^*W^*)^\top \right)^\dagger \left(\hDelta^i \hat{\Delta}^\top_i + \hDelta^i (B^*W^*)^\top  + (B^*W^*)(\hDelta^i)^\top \right)(B^*W^*(B^*W^*)^\top )^\dagger  (B^*w_{M+1}^*) \right\lvert \\
    &\leq 2 \|B^*w_m^*\|_2\|B^*w_{M+1}^*\|_2\|\left(B^*W^*(B^*W^*)^\top \right)^\dagger \left(\hDelta^i (\hDelta^i)^\top  + \hDelta^i (B^*W^*)^\top  + (B^*W^*)(\hDelta^i)^\top \right)(B^*W^*(B^*W^*)^\top )^\dagger \|_F\\
    &\leq 2R \|\left(B^*W^*(B^*W^*)^\top \right)^\dagger \left(\hDelta^i (\hDelta^i)^\top  + \hDelta^i (B^*W^*)^\top  + (B^*W^*)(\hDelta^i)^\top \right)(B^*W^*(B^*W^*)^\top )^\dagger \|_F
       \\
    &\leq \sigma  \sqrt{\epsilon_i} + \textit{noise term}(m)
\end{align*}
where the second inequality comes from the triangle inequality, the third inequality holds with probability at least $1-\delta$ by using the application of generalized inverse of matrices theorem (see Lemma~\ref{lem: inverse diff 1} for details) and the fifth inequality comes form  the fact that
$\|B^*w_{m}^*\|_2 \leq \|w_{m}^*\|_2 \leq R $. Finally, by using the assumptions of $B^*,W^*$ as well as the multi-task model estimation error $\hDelta^i $, we can apply Lemma~\ref{lem: inverse diff 2} to get the last inequality.

With the same reason, we get that, for the other direction, 
\begin{align*}
    0.5|\nu^*(m)| -|\hat{\nu}_{i+1}(m)|-\textit{noise term}(m)
    \leq \sigma \sqrt{\epsilon_i}/4 
\end{align*}

Combine these two, we have
\begin{align*}
    |\hat{\nu}_{i+1}(m)| \in \left[ 0.5 |\nu^*(m)| - \sigma \sqrt{\epsilon_i}/4 - 1.5\textit{noise term}(m)\quad , \quad 2|\nu^*(m)|  + \sigma\sqrt{\epsilon_i} + 1.5\textit{noise term}(m) \right]
\end{align*}

\textbf{----Sub-step 2 (Analyze the noise-term): } Now let's deal with the \textit{noise term}(m), we restate it below for convenience, 
\begin{align*}
        & 1.3 \left\lvert \frac{1}{n_m^i}(Z_m^i)^\top X_m^i\hat{B}_i (\hat{W}_i\hat{W}_i^\top )^\dagger   \hat{B}_i^\top  B^*w_{M+1}^* \right\lvert
        +1.3 \left\lvert (B^*w_m^*)^\top  \hat{B}_i (\hat{W}_i\hat{W}_i^\top )^\dagger   \hat{B}_i^\top  \frac{1}{n_{M+1}^i}(X_{M+1}^i)^\top Z_{M+1}^i \right\lvert\\
        &\quad +\left\lvert \frac{1}{n_m^i}(Z_m^i)^\top X_m^i\hat{B}_i (\hat{W}_i\hat{W}_i^\top )^\dagger   \hat{B}_i^\top  \frac{1}{n_{M+1}^i}(X_{M+1}^i)^\top Z_{M+1}^i \right\lvert.
\end{align*} 
By the assumption on $B^*, W^*$, it is easy to see that with high probability at least $1-\delta'$, where $\delta' = \delta/10\Gamma M$
\begin{align*}
    &\left\vert\frac{1}{n_m^i}(Z_m^i)^\top X_m^i(\hat{B}_i\hat{W}_i(\hat{B}_i\hat{W}_i)^\top )^\dagger  B^*w_{M+1}^*\right\rvert\\
    &\leq |\frac{1}{\lambda_{\min}\left(B^*W^*(B^*W^*)^\top \right)}\frac{1}{n_m^i}(Z_m^i)^\top X_m^iB^*w_{M+1}^*|\\
    &\leq \frac{\sigma}{n_m^i\lambda_{\min}\left(B^*W^*(B^*W^*)^\top \right)}
    \sqrt{(w_{M+1}^*)^\top (B^*)^\top (X_m^i)^\top X_m^iB^*w_{M+1}^*\log(1/\delta')}\\
    &\leq \frac{2.2\sigma\|w_{M+1}^*\|_2\sqrt{\log(1/\delta')}}{\sqrt{n_m^i}\lambda_{\min}\left(B^*W^*(B^*W^*)^\top \right)}\\
    &\leq \sqrt{\epsilon_i/\beta}  \frac{2.2\sigma \sqrt{R\log(1/\delta')}}{\lambda_{\min}\left(B^*W^*(B^*W^*)^\top \right)}
\end{align*}
where the first inequality comes from Lemma~\ref{lem: minimum singular value guarantee for BW}, the second the inequality comes from Chernoff inequality and the last inequality comes from the definition $n_m^i = \max\{\beta \hat{\nu}_i^2\epsilon_i^{-2},\beta \epsilon_i^{-1}, \underline{N}\}$. Note that we choose $\beta = 3000 K^2R^2(KM+Kd\log(1/\varepsilon M)+\log(M\Gamma/\delta)/\underline{\sigma}^6$. Therefore, above can be upper bounded by $\sqrt{\epsilon_i}/24$.

By the similar argument and the assumption that $n_{M+1} \geq \frac{ 3000 R \epsilon_i^{-1}}{ \underline{\sigma}^4} $, we can also show that with high probability at least $1-\delta'$
\begin{align*}
     \left\lvert (B^*w_m^*)^\top  \hat{B}_i (\hat{W}_i\hat{W}_i^\top )^\dagger   \hat{B}_i^\top  \frac{1}{n_{M+1}^i}(X_{M+1}^i)^\top Z_{M+1}^i \right\lvert
     &\leq \frac{2.2\sigma\|w_{m}^*\|_2\sqrt{\log(1/\delta')}}{\sqrt{n_{M+1}}\lambda_{\min}\left(B^*W^*(B^*W^*)^\top \right)} \\
    &\leq \sigma \sqrt{\epsilon_i}/24
\end{align*}
Finally, we have that
\begin{align*}
    \left\lvert \frac{1}{n_m^i}(Z_m^i)^\top X_m^i\hat{B}_i (\hat{W}_i\hat{W}_i^\top )^\dagger   \hat{B}_i^\top  \frac{1}{n_{M+1}^i}(X_{M+1}^i)^\top Z_{M+1}^i \right\lvert
    &\leq  \frac{2.2\sigma^2\sqrt{\epsilon_i/\beta}\|w_{m}^*\|_2\|w_{M+1}^*\|_2\log(1/\delta')}{\sqrt{n_{M+1}}\lambda_{\min}\left(B^*W^*(B^*W^*)^\top \right)} \\
    &\leq \sigma \sqrt{\epsilon_i}/24
\end{align*}

So overall we have we have $\textit{noise term}(m) \leq \sigma \sqrt{\epsilon_i}/8$.

\textbf{----Sub-step 3 (Combine the non noise-term and noise-term): }

Now when $|\nu^*(m)| \geq \sigma\sqrt{\epsilon_i}$, combine the above results show that $|\hat{\nu}_{i+1}(m)| \in \left[\nu^*(m)/16, 4|\nu^*(m)|\right]$. 

On the other hand, if $|\nu^*| \leq \sigma\sqrt{\epsilon_i}$, then we directly have $\hat{\nu}_{i+1} \in [0, 4\sigma\sqrt{\epsilon_i}]$. 

\end{proof}

\textbf{Step 2: Now we are ready to prove the following two main lemmas on the final accuracy and the total sample complexity.}



\begin{lemma}[Accuracy on each epoch]
\label{lem: main accuracy(app)}
Given $\calE_\text{source}, \calE_\text{target 1}, \calE_\text{target 2}$ and $\calE_\text{relevance}^{i,m}$ for all $m$, after the epoch $i$, we have $\textsc{ER}(\hat{B},\hat{w}_{M+1})$ upper bounded by with probability at least $1-\delta/10M\Gamma$.
\begin{align*}
    \frac{\sigma^2}{\beta} \left(KM + Kd + \log\frac{1}{\delta} \right)  s_i^* \epsilon_i^2
    + \sigma^2 \frac{\left(K + \log(1/\delta) \right)}{n_{M+1}}
\end{align*}
where $ s_i^* = \min_{\gamma \in [0, 1]} (1- \gamma) \| \nu^* \|_{0,\gamma}^i + \gamma M$ 
and $\| \nu \|_{0,\gamma}^i := |\{ m : \nu_m > \sqrt{ \gamma } \epsilon_i \}|$.

\end{lemma}

\begin{proof}
The first step is the same as proof of Theorem~\ref{thm: warm-up} in Appendix~\ref{sec: warm up}. Suppose event $\calE_\text{source}^i$ and $\calE_\text{target2}$ holds, then we have with probability at least $1-\frac{3\delta}{10 \Gamma M}$,
\begin{align*}
    \textsc{ER}(\hat{B},\hat{w}_{M+1})
    &\leq  1.3 \sigma^2 \left(KT + Kd\log( (  \sum_m n_m)/M) + \log\frac{1}{\delta} \right) \|\tilde{\nu}^*\|_2^2+\sigma^2\frac{K + \log(1/\delta)}{n_{M+1}}
\end{align*}
where $\tilde{\nu}^*(m) = \nu^*(m)/\sqrt{n_m}$. 

Now for any $\gamma \in [0,1]$, given $\calE_\text{relevance}^{i,m}$, we are going to bound $\sum_m \frac{\nu^*(m)^2}{n_m^i}$ as
\begin{align*}
    \sum_m \frac{\nu^*(m)^2}{n_m^i} 
    & \leq \sum_m \frac{\nu^*(m)^2}{n_m^i}\one\{ |\nu^*(m)| > \sigma\sqrt{\epsilon_{i-1}}\}
           + \sum_m \frac{\nu^*(m)^2}{n_m^i}\one\{ \sqrt{\gamma}\epsilon_{i-1}\leq |\nu^*(m)| \leq \sigma\sqrt{\epsilon_{i-1}}\}\\
            & \quad + \sum_m \frac{\nu^*(m)^2}{n_m^i}\one\{ |\nu^*(m)| \leq \sqrt{\gamma}\epsilon_{i}\}\\
    & \leq \sum_m \frac{256\hat{\nu}_{i}^2}{ n_m^i}\one\{ |\nu^*(m)| > \sigma\sqrt{\epsilon_{i-1}}\}
            + \sum_m  \frac{\sigma^2\epsilon_{i-1}}{n_m^i}\one\{ \sqrt{\gamma}\epsilon_{i-1}\leq |\nu^*(m)| \leq \sqrt{\epsilon_{i-1}}\}\\
            & \quad + \gamma\epsilon_i^2/\beta \sum_m \one\{ |\nu^*(m)| \leq \sqrt{\gamma}\epsilon_{i}\}\\
    &\leq \order\left(\sum_m \epsilon_i^2/\beta\one\{ |\nu^*(m)| > \sigma\sqrt{\epsilon_{i-1}}\}\right) + \order\left(\sum_m \sigma^2\epsilon_i^2/\beta\one\{ \sqrt{\gamma}\epsilon_{i-1}\leq |\nu^*(m)| \leq \sigma\sqrt{\epsilon_{i-1}}\}\right)\\
            & \quad + (M-\| \nu \|_{0,\gamma}^i)\gamma\epsilon_i^2/\beta\\
    &\leq \order\left(\sum_m \epsilon_i^2/\beta\one\{ |\nu^*(m)| > \sigma\sqrt{\gamma}\epsilon_{i-1}\right)
    + (M-\| \nu \|_{0,\gamma}^i)\gamma\epsilon_i^2/\beta\\
    &\leq  ((1-\gamma)\| \nu \|_{0,\gamma}^i + \gamma M)\epsilon_i^2/\beta
\end{align*}
\end{proof}

\begin{lemma}[Sample complexity on each epoch]
\label{lem: main complexity(app)}
Given $\calE_\text{source}, \calE_\text{target 1}$ and $\calE_\text{relevance}^{i,m}$ for all $m$, 
we have the total number of source samples used in epoch $i$ as as
\begin{align*}
    \order\left(\beta (M\varepsilon^{-1} +  \|\nu^*\|_2^2\varepsilon^{-2})\right)
\end{align*}
\end{lemma}
\begin{proof}
Given $\calE_\text{relevance}^{i,m}$, we can get the sample complexity for block $i$ as the follows.
\begin{align*}
    \sum_m^M n_m^i
    & = \sum_m^M  \max\{\hat{\nu}_i(m)^2\epsilon_i^{-2},\epsilon_i^{-1}, \underline{N}\}\\
    & \leq \sum_m^M \beta \hat{\nu}_i^2\epsilon_i^{-2}
        + \sum_m^M \beta\epsilon_i^{-1}\\
    & \leq \sum_m^M \beta\hat{\nu}_i^2(m)\epsilon_i^{-2} \one\{ |\nu^*(m)| > \sigma \sqrt{\epsilon_{i-1}} \}
        + \sum_m^M \beta\hat{\nu}_i^2(m)\epsilon_i^{-2} \one\{ |\nu^*(m)| \leq \sigma \sqrt{\epsilon_{i-1}}\}
        + \sum_m^M \beta\epsilon_i^{-1} \\
    & \leq \sum_m^M \beta(4\nu^*(m))^2\epsilon_i^{-2} \one\{ |\nu^*(m)| > \sigma \sqrt{\epsilon_{i-1}} \}
        + \sum_m^M \beta(4\sigma\sqrt{\epsilon_{i-1}})^2/\epsilon_i^{-2}  \one\{ |\nu^*(m)| \leq \sigma \sqrt{\epsilon_{i-1}}\}
        + \sum_m^M \beta\epsilon_i^{-1} \\
    & = \order\left(\beta (M\epsilon_i^{-1} +  \|\nu^*\|_2^2\epsilon_i^{-2})\right)
\end{align*}
\end{proof}

\begin{theorem}
\label{thm: main}
Suppose we know in advance a lower bound of $\sigma_{\min}(W^*)$ denoted as $\underline{\sigma}$.
Under the benign low-dimension linear representation setting as defined in Assumption~\ref{assump: identity covar}, we have $\textsc{ER}(\hat{B},\hat{w}_{M+1}) \leq \varepsilon^2$ with probability at least $1-\delta$ whenever the number of source samples $N_{total}$ is at least
\begin{align*}
     &\otil\bigg(\left(K(M+d) + \log\frac{1}{\delta} \right) \sigma^2  s^* \|\nu^*\|_2^2 \varepsilon^{-2} 
     + \square \sigma \varepsilon^{-1} \bigg)
\end{align*}
where $\square = \left(MK^2dR/\underline{\sigma}^3\right)\sqrt{s^*}$
and the target task sample complexity $n_{M+1}$ is at least
\begin{align*}
    &\otil\left(\sigma^{2}K\varepsilon^{-2} +\Diamond\sqrt{s^*} \sigma \varepsilon^{-1}\right)
\end{align*}
where $\Diamond = \min\left\{\frac{\sqrt{R}}{\underline{\sigma}^2K} ,\sqrt{K(M+d)+\log\frac{1}{\delta} }\right\}$ and
 $s^*$ has been defined in Theorem~\ref{thm: warm-up}. 
 
\end{theorem}

\begin{proof}
Given $\calE_\text{source}, \calE_\text{target 1},\calE_\text{target 2}$ and $\calE_\text{relevance}^{i,m}$ then by Lemma~\ref{lem: main accuracy(app)}, we the final accuracy from last epoch as 
\begin{align*}
    \textsc{ER}_{M+1}(\hat{B},\hat{w}_{M+1})
    \leq \sigma^2 \left(KM + Kd\log(  (\sum_m n_m)/M) + \log\frac{1}{\delta} \right)  s_\Gamma^* \epsilon_\Gamma^2/\beta + \sigma^2\frac{\left(K + \log(1/\delta) \right)}{n_{M+1}}
\end{align*}
Denote the final accuracy of the first term as $\varepsilon^2$. So we can write $\epsilon_\Gamma$ as
\begin{align*}
    \varepsilon/ \left(\sigma \sqrt{KM + Kd\log(  (\sum_m n_m)/M) + \log\frac{1}{\delta} } \sqrt{s_\Gamma^*/\beta}\right)
\end{align*}
By applying lemma~\ref{lem: main complexity}, we requires the total source sample complexity 
\begin{align*}
    \sum_{i=1}^\Gamma \beta (M\epsilon_i^{-1} +  \|\nu^*\|_2^2\epsilon_i^{-2}) 
    & \leq 2\beta (M\epsilon_\Gamma^{-1} +  2\|\nu^*\|_2^2\epsilon_\Gamma^{-2}) \\
    & = \beta M \sigma \sqrt{KM + Kd\log(  (\sum_m n_m)/M) + \log\frac{1}{\delta} } \sqrt{s_\Gamma^*/\beta} \varepsilon^{-1}\\
        &\quad + \beta \|\nu^*\|_2^2 \sigma^2 \left( KM + Kd\log(  (\sum_m n_m)/M) + \log\frac{1}{\delta}\right)s_\Gamma^*\varepsilon^{-2}/\beta \\
    & = \sqrt{\beta} M \sqrt{s_\Gamma^*} \sigma \sqrt{KM + Kd\log(  (\sum_m n_m)/M) + \log\frac{1}{\delta} } \varepsilon^{-1}\\
        &\quad +  s_\Gamma^* \sigma^2 \left( KM + Kd\log(  (\sum_m n_m)/M) + \log\frac{1}{\delta}\right) \|\nu^*\|_2^2 \varepsilon^{-2}\\
    & = \otil\left(\left(MK^2d + M\sqrt{Kd}/\underline{\sigma}^2\right)\sqrt{s_\Gamma^*}  \sigma \varepsilon^{-1} +Kds_\Gamma^* \sigma^2 \|\nu^*\|_2^2 \varepsilon^{-2}\right)
\end{align*}
Also in order to satisfy the assumption in Lemma~\ref{lem: nu estimation error(app)}, we required $n_{M+1}$ to be at least
\begin{align*}
    \frac{\varepsilon^{-1}}{\underline{\sigma}^2}
    & = \frac{1}{\underline{\sigma}^2}\sqrt{s_\Gamma^*/\beta} \sigma \sqrt{KM + Kd\log(  (\sum_m n_m)/M) + \log\frac{1}{\delta} } \varepsilon^{-1} \\
    &\leq \min\left\{\frac{1}{\underline{\sigma}^2K}, \sqrt{KM + Kd\log(  (\sum_m n_m)/M) + \log\frac{1}{\delta} }\right\}\sqrt{s_\Gamma^*} \sigma \varepsilon^{-1}
\end{align*}

Notice that $\Gamma$ is a algorithm-dependent parameter, therefore, the final step is to bound $s_\Gamma^*$ by an algorithm independent term by write $\Gamma$ as
\begin{align*}
    \Gamma 
    = -\log \epsilon_\Gamma
    \leq  \min\left\{ \log\sqrt{\frac{N_{total}}{\beta\|\nu^*\|_2^2}},\log\frac{N_{total}}{\beta M} \right\}
\end{align*}
So we have .
\begin{align*}
    \| \nu \|_{0,\gamma}^\Gamma
    = |\{ m : \nu_m > \sqrt{ \gamma } \max\{\frac{\beta\|\nu^*\|_2^2}{N_{total}},\frac{\beta M}{N_{total}} \}|
\end{align*}
To further simply this, notice that, for any $\epsilon' < \epsilon_i$.
\begin{align*}
    \| \nu \|_{0,\gamma}^i 
    = |\{ m : \nu_m > \sqrt{ \gamma } \epsilon_i \}|
    \leq |\{ m : \nu_m > \sqrt{ \gamma } \epsilon' \}|
\end{align*}
So we further have
\begin{align*}
    \| \nu \|_{0,\gamma}^\Gamma
    = |\{ m : \nu_m > \sqrt{ \gamma } \frac{\|\nu^*\|_2^2}{N_{total}}\}|
    : = \| \nu \|_{0,\gamma}
\end{align*}

Finally, by union bound $\calE_\text{source}, \calE_\text{target 1},\calE_\text{target 2}$ and $\calE_\text{relevance}^{i,m}$ on all epochs, we show that all the lemmas holds with probability at least $1-\delta$.

\end{proof}

\subsection{Auxiliary Lemmas}

\begin{lemma}[Convergence on estimated model $\hat{B}_i\hat{W}_i$]
\label{cor: diff of estimation BW}
For any fixed $i$, given $\calE_\text{source}^i $, we have
\begin{align*}
    \|\hDelta^i \|_F^2 
    & \leq 1.3\sigma^2\left(KM + Kd\log((\sum_m n_m^i)/M) + \log\frac{10 \Gamma}{\delta} \right) \epsilon_i/\beta
\end{align*}
And therefore, when $\beta = 3000 K^2R^2(KM+Kd\log(N_\text{total}/M)+\log(M\Gamma/\delta)/\underline{\sigma}^6$. Therefore, above can be upper bounded by $\sqrt{\epsilon_i}/24$,

we have $\|\Delta\|_F^2 \leq \frac{\sigma^2\epsilon_i}{4K^2R^2}$

\end{lemma}
\begin{proof}
Denote $\Delta_m$ as the $m$-th column of $\hDelta^i $.
\begin{align*}
    \sum_{m=1}^M \|X_m \Delta_m \|_2^2
    & = \sum_{m=1}^M \Delta_m^\top  X_m^\top  X_m \Delta_m\\
    & \geq 0.9 \sum_{m=1}^M n_m \Delta_m^\top  \Delta_m \\
    & \geq 0.9 \min_m n_m^i \sum_{m=1}^M \|\Delta_m\|_2^2 
    = 0.9 \min_m n_m^i \|\hDelta^i \|_F^2
\end{align*}
Recall that our definition on $n_m^i = \max\left\{\beta \hat{\nu}_{i}^2(m)\epsilon_i^{-2},\beta \epsilon_i^{-1}\right\}$  and also use the upper bound derived in Claim~\ref{claim 1}, we finish the proof.
\end{proof}

\begin{lemma}[minimum singular value guarantee for $\hat{B}_i\hat{W}_i$]
\label{lem: minimum singular value guarantee for BW}
For all $i$, we can guarantee that
\begin{align*}
    \sigma_{min}(\hat{B}_i\hat{W}_i) \geq \sigma_{min}(W^*)/2
\end{align*}
\end{lemma}
Also because $M \geq K$, so there is always a feasible solution for $\hat{\nu}_i$.
\begin{proof}
Because $B^*$ is a orthonormal, so $\sigma_{min}(B^*W^*) = \sigma_{min}(W^*)$. Also from Lemma~\ref{cor: diff of estimation BW} and Weyl's theorem stated below, we have
$|\sigma_{\min}(\hat{B}_i\hat{W}_i)) - \sigma_{min}(B^*W^*)| \leq \|\hat{B}_i\hat{W}_i- B^*W^*\|_F \leq \frac{\underline{\sigma}}{2} \leq \frac{\sigma_{min}(W^*)}{2}$. Combine these two inequality we can easily get the result.

\begin{theorem}[Weyl's inequality for singular values]
\label{thm: Weyl}
Let $M$ be a $p \times n$ matrix with $1 \leq p \leq n$. Its singular values $\sigma_{k}(M)$ are the $p$ positive eigenvalues of the $(p+n) \times(p+n)$ Hermitian augmented matrix
$$
\left[\begin{array}{cc}
0 & M \\
M^{*} & 0
\end{array}\right]
$$
Therefore, Weyl's eigenvalue perturbation inequality for Hermitian matrices extends naturally to perturbation of singular values. This result gives the bound for the perturbation in the singular values of a matrix $M$ due to an additive perturbation $\Delta$ :
$$
\left|\sigma_{k}(M+\Delta)-\sigma_{k}(M)\right| \leq \sigma_{1}(\Delta) \leq \|\Delta\|_F
$$
\end{theorem}
\end{proof}
 
\vspace{30px}

\begin{lemma}
\label{lem: KKT for full rank}
For any two matrix $M_1 \in \fR^{K \times M}, M_2 \in \fR^{K}$ where $K\leq M$. Suppose $\text{rank}(M_1) = K$ and define $\Tilde{\nu}$ as 
\begin{align*}
    \argmin_{\nu \in \fR^M} \|\nu\|_2^2 
    \quad \text{s.t. }
    M_1 \nu = M_2, \nu \in \fR^m,
\end{align*}
then we have 
\begin{align*}
    \Tilde{\nu} = M_1^\top (M_1M_1^\top )^{-1}M_2
\end{align*}
\end{lemma}
\begin{proof}
    We prove this by using KKT conditions,
    \begin{align*}
         L(\nu, \lambda) = \|\nu\|_2^2 + \lambda^\top (M_1\nu_1-M_2)
    \end{align*}
    Given $0 \in \partial L$, we have $\Tilde{\nu} = - (M_1)^\top \lambda/2$. Then by replace this into the constrains, we have
    \begin{align*}
        M_1M_1^\top  \lambda = -2M_2
        \to
        \lambda = -2 \left( M_1M_1^\top  \right)^{-1}M_2
    \end{align*}
    and therefore $\Tilde{\nu} = (M_1)^\top  \left( M_1M_1^\top  \right)^{-1}M_2$. 
\end{proof}
 
\begin{lemma}[A closed form expression for $\hat{\nu}_i$]
\label{lem: a closed form for estimated nu}
For any epoch $i$, given the estimated representation $\hat{B}_i$, we have
    \begin{align*}
    &\hat{\nu}_{i+1}
    =\argmin_\nu \|\nu\|_2^2  \\
    \text{s.t.} \quad
    & \sum_m  \alpha_m^i \left(\hat{\Sigma}_m^i B^*w_m^* 
        +\frac{1}{n_m^i} \left(X_m^i\right)^\top Z_m\right)\nu(m)
        =  \alpha_{M+1}^i \left(\hat{\Sigma}_M^i B^*w_{M+1}^*  +\frac{1}{n_{M+1}^i} \left(X_{M+1}^i\right)^\top Z_{M+1}\right)
    \end{align*}
    where $\alpha_m^i = \left( \hat{B}_i^\top \hat{\Sigma}_m^i\hat{B}_i \right)^{-1}\hat{B}_i^\top $
\end{lemma}
\begin{proof}
For any epoch $i$ and it's estimated representation $\hat{B}_i$, by least square argument, we have
\begin{align*}
    \hat{w}_m^i 
    & = \argmin_w \|X_m^i\hat{B}_i w - Y_m\|_2\\
    & = \left( \left(X_m^i\hat{B}_i\right)^\top  X_m^i\hat{B}_i\right)^{-1}\left(X_m^i\hat{B}_i\right)^\top Y_m\\
    &= \left( \left(X_m^i\hat{B}_i\right)^\top  X_m^i\hat{B}_i\right)^{-1}\left(X_m^i\hat{B}_i\right)^\top X_m^iB^*w_m^* 
        + \left( \left(X_m^i\hat{B}_i\right)^\top  X_m^i\hat{B}_i\right)^{-1}\left(X_m^i\hat{B}_i\right)^\top Z_m\\
    &=\underbrace{\left( \hat{B}_i^\top \hat{\Sigma}_m^i\hat{B}_i\right)^{-1}\hat{B}_i^\top }_{\alpha_m^i \in \fR^{k \times d}}\hat{\Sigma}_m^iB^*w_m^*
        + \underbrace{\left( \hat{B}_i^\top \hat{\Sigma}_m^i\hat{B}_i \right)^{-1}\hat{B}_i^\top }_{\alpha_m^i}\left(X_m^i\right)^\top Z_m
\end{align*}

Therefore, combine this with the previous optimization problem, we have
which implies that
\begin{align*}
    &\hat{W}_i\nu
    =\sum_m  \alpha_m^i \left(\hat{\Sigma}_m^iB^*w_m^* 
        +\frac{1}{n_m^i} \left(X_m^i\right)^\top Z_m\right)\nu(m)\\
    &\hat{w}_{M+1}^i
    = \alpha_{M+1}^i \left(\hat{\Sigma}_M^i B^*w_{M+1}^*  +\frac{1}{n_{M+1}^i} \left(X_{M+1}^i\right)^\top Z_{M+1}\right)
\end{align*}
Recall the definition of $\hat{\nu}_{i+1}$ as \begin{align*}
    \min \|\nu\|_2^2  
    \quad \text{s.t.}
    \quad \hat{W}_i\nu = \hat{w}_{M+1}^i
\end{align*}
Therefore we get the closed-form by replace $\hat{W}_i\nu, \hat{w}_{M+1}^i$ with the value calculated above.
\end{proof}

\begin{lemma}[Difference of the inverse covariance matrix]
\label{lem: inverse diff 1}
For any fixed $i,m$ and any proper matrices $M_1,M_2,M_3,M_4$,
\begin{align*}
    &\left\lvert M_1\left((\hat{B}_i\hat{W}_i(\hat{B}_i\hat{W}_i)^\top )^\dagger  - \left(B^*W^*(B^*W^*)^\top \right)^{+}\right) B^*M_2 \right\rvert\\
    &= \|M_1\|\|\left(B^*W^*(B^*W^*)^\top \right)^\dagger \left(\hat{\Delta_i}(\hDelta^i)^\top  + \hDelta^i (B^*W^*)^\top  + (B^*W^*)(\hDelta^i)^\top \right)(B^*W^*(B^*W^*)^\top )^\dagger \|\|B^*M_2\|\\
    &\left\lvert M_3(B^*)^\top \left((\hat{B}_i\hat{W}_i(\hat{B}_i\hat{W}_i)^\top )^\dagger  - \left(B^*W^*(B^*W^*)^\top \right)^{+}\right)M_4 \right\rvert\\
    &= \|M_3(B^*)^\top \|\|\left(B^*W^*(B^*W^*)^\top \right)^\dagger \left(\hDelta^i (\hDelta^i)^\top  + \hat{\Delta}(B^*W^*)^\top  + (B^*W^*)(\hDelta^i)^\top \right)(B^*W^*(B^*W^*)^\top )^\dagger \|\|M_4\|
\end{align*}
\end{lemma}
\begin{proof}
First we want to related these two inverse term,
\begin{align*}
    &(\hat{B}_i\hat{W}_i\hat{W}_i^\top \hat{B}_i^\top )^\dagger  - \left(B^*W^*(B^*W^*)^\top \right)^\dagger  \\
    &\leq  \left((B^*W^* + \hDelta^i )(B^*W^* + \hDelta_i)^\top \right)^\dagger  - \left(B^*W^*(B^*W^*)^\top \right)^\dagger  \\
    &\leq  \left(B^*W^*(B^*W^*)^\top  + \left(\hDelta^i (\hDelta^i)^\top  + \hDelta^i (B^*W^*)^\top  + (B^*W^*)(\hDelta^i)^\top \right)\right)^\dagger  - \left(B^*W^*(B^*W^*)^\top \right)^\dagger  
\end{align*}    
In order to connect the first pseudo inverse with the second, we want to use the generalized inverse of matrix theorem as stated below.
\begin{theorem}[Theorem from \cite{Kovanic1979}]
\label{them: pseudo-inverse of matrics sum}
If $V$ is an $n \times n$ symmetrical matrix and if $X$ is an $n \times q$ arbitrary real matrix, then
\begin{align*}
    (V+XX^\top )^\dagger 
    = V^\dagger  - V^\dagger X(I+X^\top V^\dagger X)^{-1}X^\top V^\dagger  + ((X_\perp)^\dagger )^\top X_\perp^\dagger 
\end{align*}
where $X_\perp = (I - VV^\dagger )X$
\end{theorem}
It is easy to see that
$V: = B^*W^*(B^*W^*)^\top $
and we can also decompose 
$\left(\hat{\Delta}\hat{\Delta}^\top  + \hat{\Delta}(B^*W^*)^\top  + (B^*W^*)\hat{\Delta}^\top \right)$ into some $XX^\top $,

Therefore, we can write the above inequality as 
\begin{align*}
    - V^\dagger X(I+X^\top V^\dagger X)^{-1}X^\top V^\dagger  + ((X_\perp)^\dagger )^\top X_\perp^\dagger 
\end{align*}

Next we show that
$((X_\perp)^\dagger )^\top X_\perp^\dagger  B^*= 0$ and $(B^*)^\top ((X_\perp)^\dagger )^\top X_\perp^\dagger  = 0$. Let $U D V^\top $ be the singular value decomposition of $B^*W^*$.
So we have
\begin{align*}
    VV^\dagger 
    = UD^2U^\top  (UD^2U^\top )^\dagger 
    = UU^\top ,
\end{align*}
and therefore, 
because $B^*$ are contained in the column spaces as $B^*W^*$,
\begin{align*}
    X_\perp^\dagger B^*
    & = (U_\perp U_\perp^\top  X)^\dagger B^*\\
    & = X^\dagger  U_\perp U_\perp^\top B^*
    = 0.
\end{align*}
Therefore, we conclude that
\begin{align*}
    & \left\lvert M_1\left((\hat{B}_i\hat{W}_i(\hat{B}_i\hat{W}_i)^\top )^\dagger  - \left(B^*W^*(B^*W^*)^\top \right)^{+}\right) B^*M_2 \right\rvert\\
    & \leq \|M_1\|\|\underbrace{\left(B^*W^*(B^*W^*)^\top \right)^\dagger \left(\hDelta(\hDelta^i)^\top  + \hDelta(B^*W^*)^\top  + (B^*W^*)(\hDelta^i)^\top \right)(B^*W^*(B^*W^*)^\top )^\dagger }_{V^\dagger XX^\top V^\dagger }\|\|B^*M_2\|
\end{align*}
and so does the other equation.
\end{proof}

\vspace{30px}

\begin{lemma}
\label{lem: inverse diff 2}
    Given $\calE_\text{source}^i $, for any fixed $i$,
    we have
    \begin{align*}
        &\|  \left(B^*W^*(B^*W^*)^\top \right)^\dagger \left(\hDelta^i (\hDelta^i)^\top  + \hat{\Delta}(B^*W^*)^\top  + (B^*W^*)(\hDelta^i)^\top \right)(B^*W^*(B^*W^*)^\top )^\dagger \|_F\\
        &\leq 3\sigma\|(W^*)^\dagger \|_F\|(W^*(W^*)^\top )^\dagger \|_F
        \underline{\sigma}^3\sqrt{\epsilon_i}/6KR\\
        &\leq \sigma \sqrt{\epsilon_i}/2R
    \end{align*}
\end{lemma}
\begin{proof}
The target term can be upper bounded by
\begin{align*}
    &\|  \left(B^*W^*(B^*W^*)^\top \right)^\dagger \hDelta^i (\hDelta^i)^\top (B^*W^*(B^*W^*)^\top )^\dagger \|_F \\
        &\quad + \|  \left(B^*W^*(B^*W^*)^\top \right)^\dagger B^*W^*(\hDelta^i)^\top (B^*W^*(B^*W^*)^\top )^\dagger \|_F\\
        &\quad + \|  \left(B^*W^*(B^*W^*)^\top \right)^\dagger \hDelta^i (B^*W^*)^\top (B^*W^*(B^*W^*)^\top )^\dagger \|_F
\end{align*}

Before we do the final bounding, we first show the upper bound of the following term, which will be used a lot,

\begin{align*}
    \| (B^*W^*)^\top \left(B^*W^*(B^*W^*)^\top \right)^\dagger \|_F
    &=\|  \left(B^*W^*(B^*W^*)^\top \right)^\dagger B^*W^* \|_F \\
    & = \| (B^*W^*(W^*)^\top (B^*)^\top )^\dagger B^*W^* \|_F\\
    & = \| ((B^*)^\top )^\dagger (W^*(W^*)^\top )^\dagger (B^*)^\dagger B^*W^* \|_F\\
    & = \|B^* (W^*(W^*)^\top )^\dagger W^*\|_F\\
    &= \|(W^*(W^*)^\top )^\dagger W^*\|_F\\
    &= \|(W^*)^\top )^\dagger (W^*)^\dagger W^*\|_F\\
    & = \|((W^*)^\dagger W^*)^\top (W^*)^\dagger \|_F\\
     & = \|(W^*)^\dagger W^*(W^*)^\dagger \|_F\\
    & = \|((W^*)^\dagger \|_F
\end{align*}
Therefore, we can bound the whole term by
\begin{align*}
    \|(B^*W^*(B^*W^*)^\top )^\dagger \|_F^2\|\hDelta^i \|_F^2 
        + 2\|(W^*)^\dagger \|_F\|\hDelta^i \|_F\|(B^*W^*(B^*W^*)^\top )^\dagger \|_F
    \leq 3\|(W^*)^\dagger \|\|(W^*(W^*)^\top )^\dagger \|_F\|\hDelta^i \|_F
\end{align*}

Recall that we have $\|\hDelta^i \|_F^2 \leq \frac{\sigma^2 \underline{\sigma}^6 \epsilon_i}{36K^2R^2}$ given $\calE_\text{source}^i $
therefore, we get the final result.
\end{proof}

\newpage
\section{Experiment details}
\label{sec: experiment(app)}

\subsection{Other implementation details}
\label{subesc: implement details}
We choose $\beta_i = 1/\|\nu\|_2^2$, which is usually $\Theta(1)$ in practice. Instead of choosing $\epsilon_i = 2^i$, we set that as $1.5^{-i}$ and directly start from $i = 22$. It is easy to see that it turns out the actual sample number used in the experiment is similar as choosing $\beta = \text{poly}(d,K,M)$ and start from epoch $1$ as proposed in theorem. But our choice is more easy for us to adjust parameter and do comparison.

We run each experiment on each task only once due to the limited computational resources and admitted that it is better to repeat the experiment for more iterations. But since we have overall $160$ target tasks, so considering the randomness among tasks, we think it still gives meaningful result. 

Moreover, remember that we have lower bound $\beta \epsilon_i^{-1}$ in our proposed algorithm, In our experiment for linear model, we actually find that only using a constant small number like $50$ for each epoch is enough for getting meaningful result. While in convnet, considering the complexity of model, we still obey this law.

\subsection{More results and analysis for linear model}
\label{subesc: more on linear model}

\subsubsection{More comprehensive summary}
\begin{figure}[ht]
    \centering
    \includegraphics[scale=0.55]{plots/linear_result.pdf}
    \includegraphics[scale=0.55]{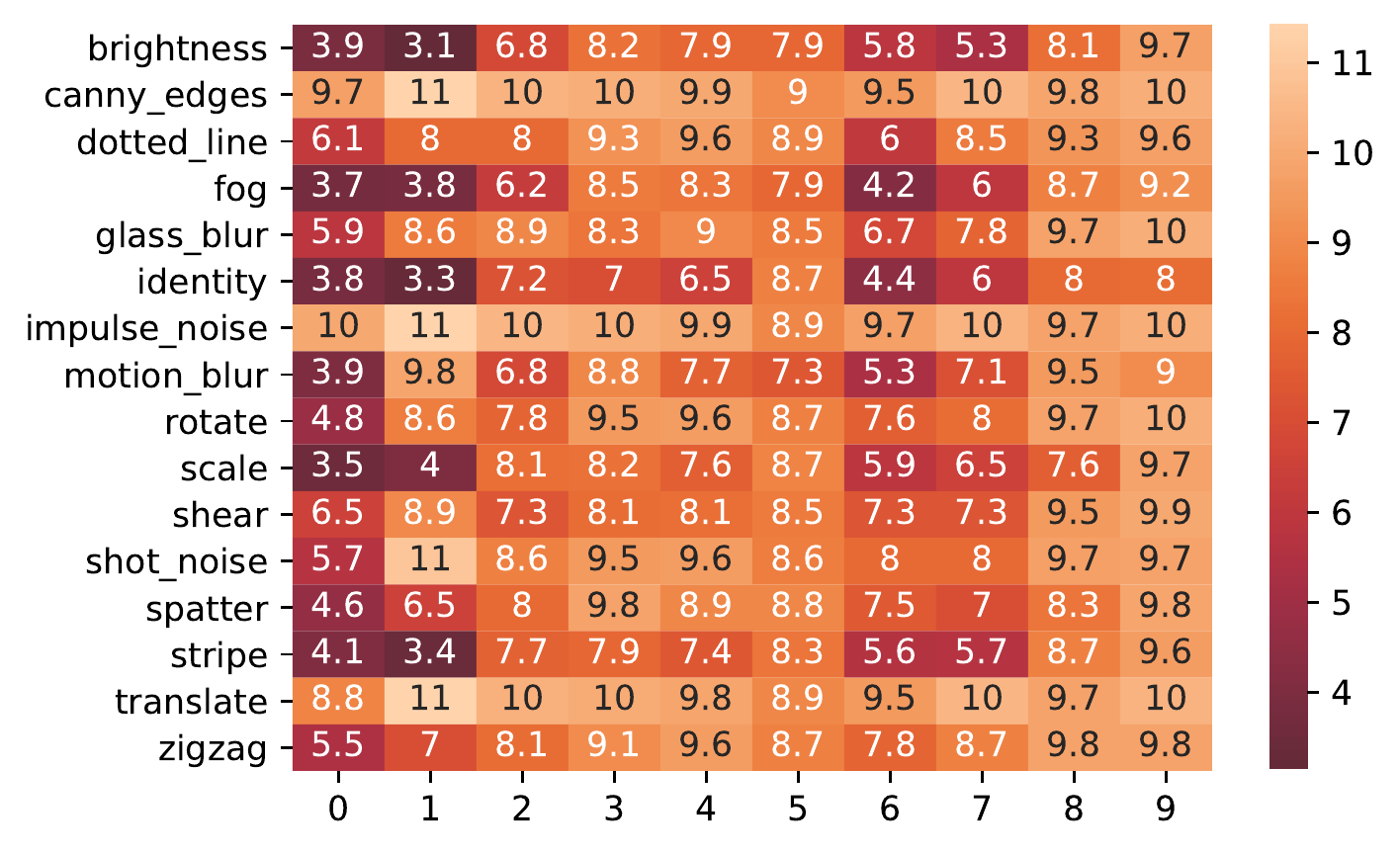}
    \caption{ \textbf{summary of performance difference for linear model (restated of figure~\ref{fig: linear summary}); \textbf{left:} The prediction difference (in \%) between ada and non-ada for all target tasks \textbf{right}: The incorrect percentage of non-adaptive algorithm.Note that $10\%$ is the baseline due to the in-balance dataset and the large the worse. Please refer to the main paper for further explanation},
    }
    \label{fig: linear summary (app)}
\end{figure}

\subsubsection{Why our algorithm fails on some target tasks ?}
\label{subsec: bad examples on linear (app)}
Here we focus on why our algorithm get bad performance on some tasks. Overall, other than the randomness, this is mainly due to the incompatibility between our theoretical assumption (realizable linear model) and the complicated data structure in reality. 

To be specific, there might a subset of source tasks that are informative about the target task under linear model assumptions, but other source tasks are far away from this model assumption. Due to the model misspecification, those misleading tasks may gain more sampling in the adaptive algorithm. We conjecture that this might be the case for target tasks like \textit{scale\_0} and \textit{scale\_2}. To further support our argument, we further analyze its sample number distribution and test error changing with increasing epoch in the next paragraph.

For the \textit{scale\_0} task, in Figure~\ref{fig: bad weight} we observe the non-stationary sample distribution changing across each epoch. But fortunately, the distribution is not extreme, there are still some significant sample from \textit{X\_0} source tasks. This aligns with the test error observation in Figure~\ref{fig: bad test error}, which is still gradually decreasing, although slower than the non-adaptive algorithm. On the other hand, the sample distributions are even worse. We observed that nearly all the samples concentrate towards \textit{X\_5} source tasks. Thus no only we can not sample enough informative data, but we also force the model to fit to some non-related data. Such mispecifcation has been reflected in the test error changing plot. (You may notice the unstable error changing in non-adaptive algorithm performance, we think it is acceptable randomness because we only run each target task once and the target task itself is not very easy to learn.)

\begin{figure}[ht]
    \centering
    \includegraphics[scale=0.22]{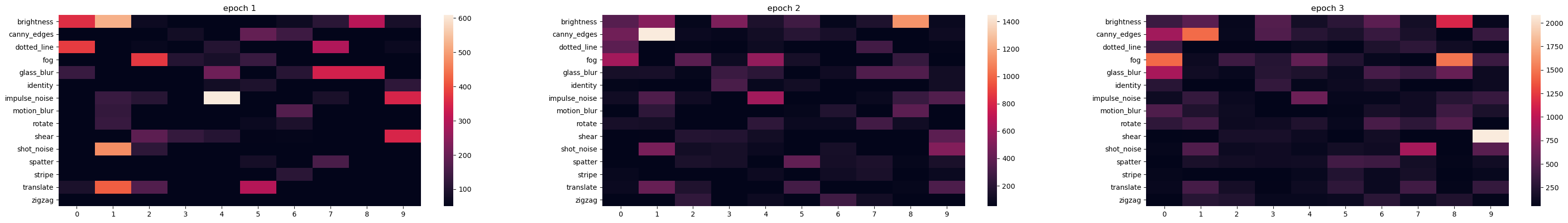}
    \includegraphics[scale=0.22]{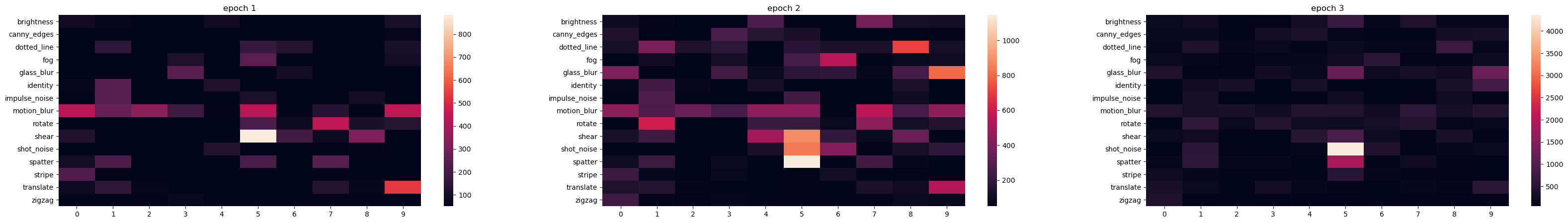}
    \caption{ \textbf{top: sample distribution for target task as \textit{scale\_0}, bottom: sample distribution for target task as \textit{scale\_2}}
    \small We show the sample distribution at each epoch 1,2,3.
    }
    \label{fig: bad weight}
\end{figure}

\begin{figure}[ht]
    \centering
    \includegraphics[scale=0.4]{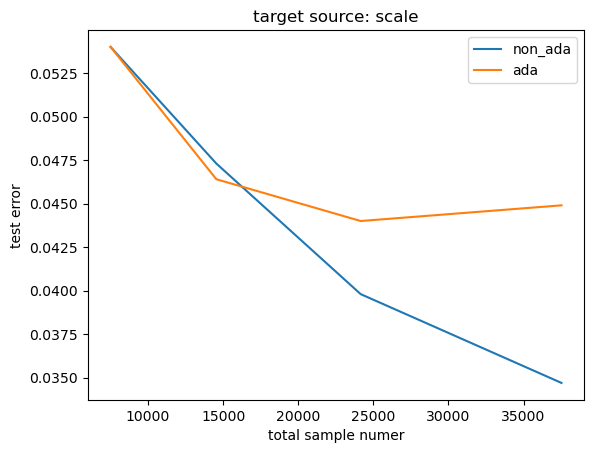}
    \includegraphics[scale=0.4]{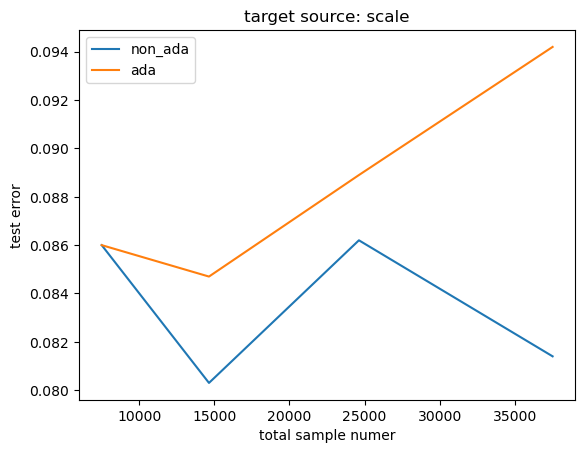}
    \caption{ \textbf{Test error change after each epoch for target task as \textit{scale\_0} and \textit{scale\_2}}
    }
    \label{fig: bad test error}
\end{figure}

\subsubsection{More good sample distribution examples}
Appendix~\ref{subsec: good example linear}
\begin{figure}[H]
    \centering
    \includegraphics[scale=0.22]{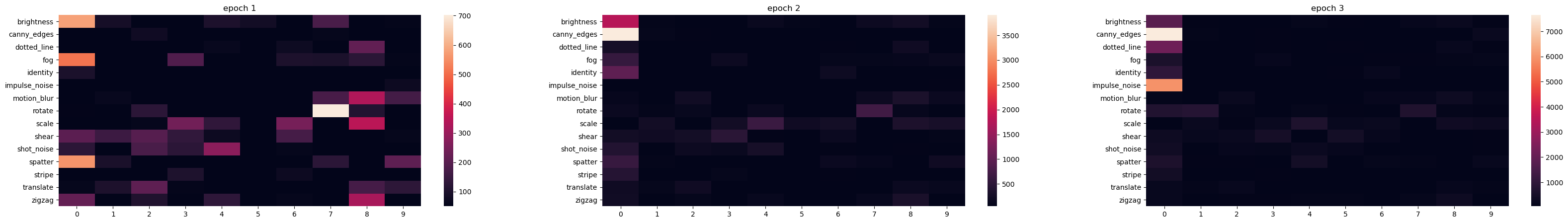}
    \includegraphics[scale=0.22]{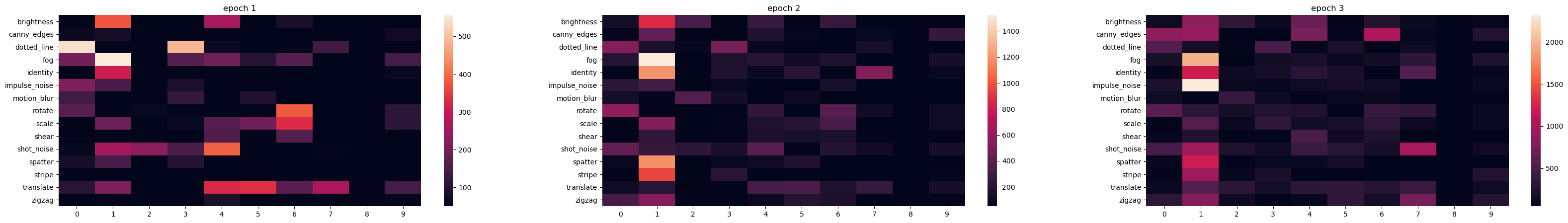}
    \includegraphics[scale=0.22]{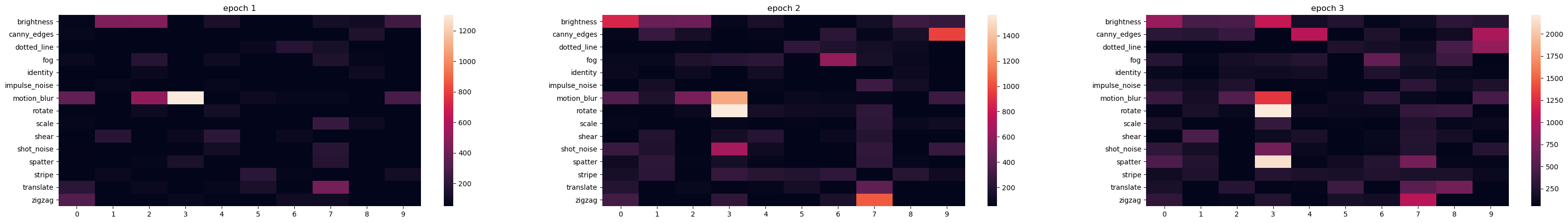}
    \includegraphics[scale=0.22]{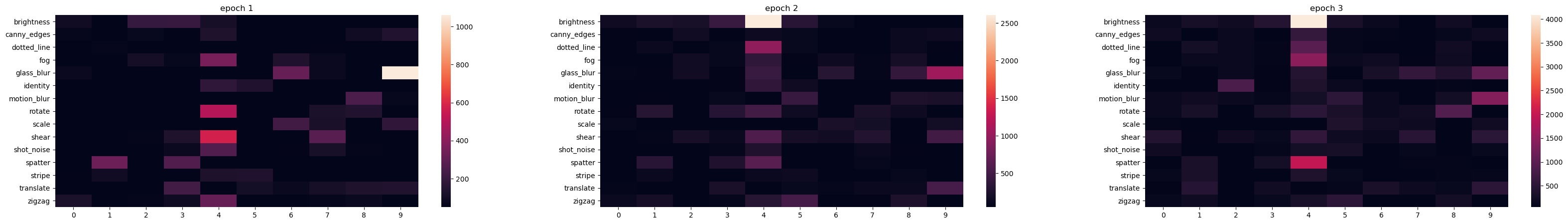}
    \includegraphics[scale=0.22]{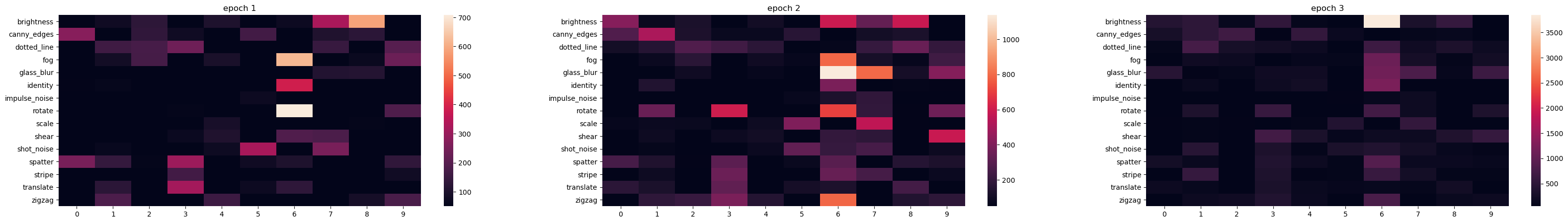}
    \includegraphics[scale=0.22]{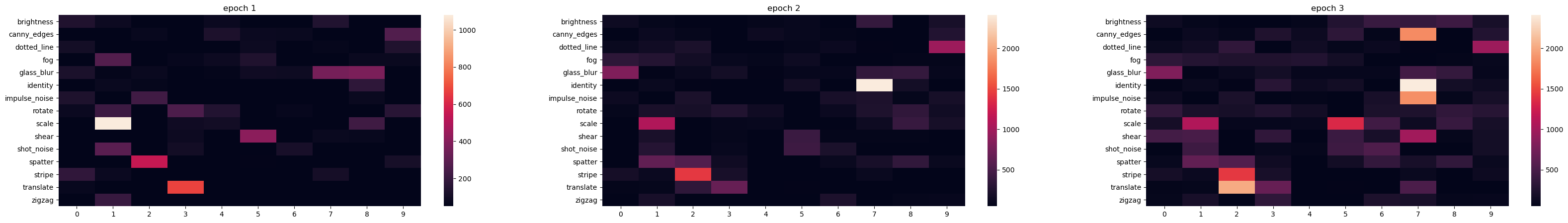}
    \includegraphics[scale=0.22]{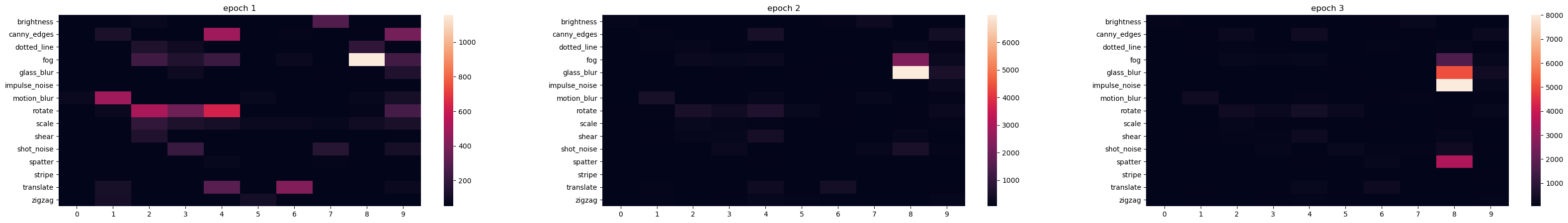}
    \includegraphics[scale=0.22]{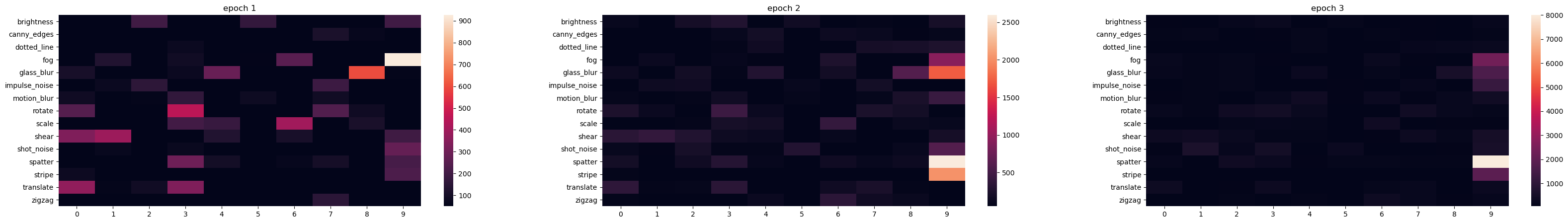}
    \caption{\textbf{Good sample distribution.}  We show the sample distribution at each epoch 1,2,3. \textbf{From top to bottom:} \textit{glass\_blur\_0},\textit{glass\_blur\_1},\textit{glass\_blur\_3},\textit{impulse\_noise\_4},\textit{motion\_blur\_6},\textit{motion\_blur\_7},\textit{identity\_8},\textit{identity\_9}}
\end{figure}

\newpage
\subsection{More results and analysis for convnet model}
\label{subsec: good example linear}
The convnet gives overall better accuracy than the linear model, except the translate class, as shown in Figure~\ref{fig: cnn summary (app)}. So we want to argue that it might be harder for us to get as large improvement as on linear model given the better expressive power on convnet.
\begin{figure}[ht]
    \centering
    \includegraphics[scale=0.55]{plots/cnn_result.pdf}
    \includegraphics[scale=0.55]{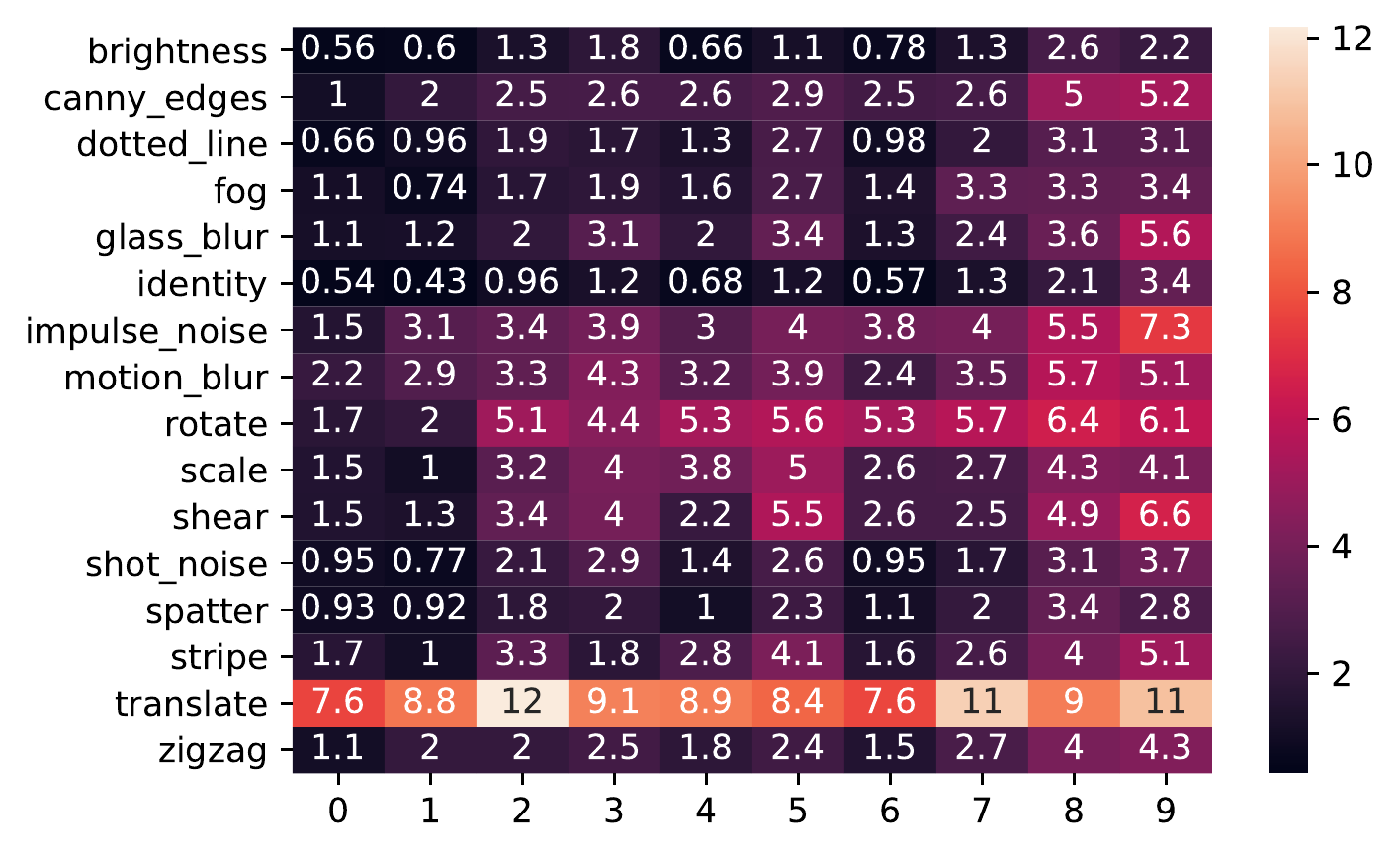}
    \caption{ \textbf{left: summary of performance difference for conv model (restated of figure~\ref{fig: cnn_summary}); right: the incorrect percentage of non-adaptive algorithm}
    \small On the right side we show the incorrect percentage of non-adaptive algorithm. Note that $10\%$ is the baseline due to the in-balance dataset and the large the worse. Please refer to the main paper for further explanation,
    }
    \label{fig: cnn summary (app)}
\end{figure}

\subsubsection{Why our algorithm fails on some target tasks ?}
\label{subsec: bad examples on cnn (app)}
Here we show \textit{scale\_5} and \textit{shear\_9} as the representative bad cases. With the similar idea of linear model, we again observe the non-stationary sample distribution changing across each epoch in Figure~\ref{fig: bad weight}. For \textit{scale\_5}, we observe that at the beginning of epoch 1, the sample fortunately converges to \textit{X\_5} source tasks, therefore our adaptive algorithm initially performs better than the non-adaptive one as shown in Figure~\ref{fig: bad test error}. Unfortunately, the sample soon diverges to other source tasks, which more test error. For \textit{shear\_9}, although there are some samples concentrate on \textit{X\_9} source tasks, overall, the number of samples on \textit{X\_9} source tasks has a decrease proportion of total number of source sample. So the algorithm has a worse performance on this.

\begin{figure}[ht]
    \centering
    \includegraphics[scale=0.22]{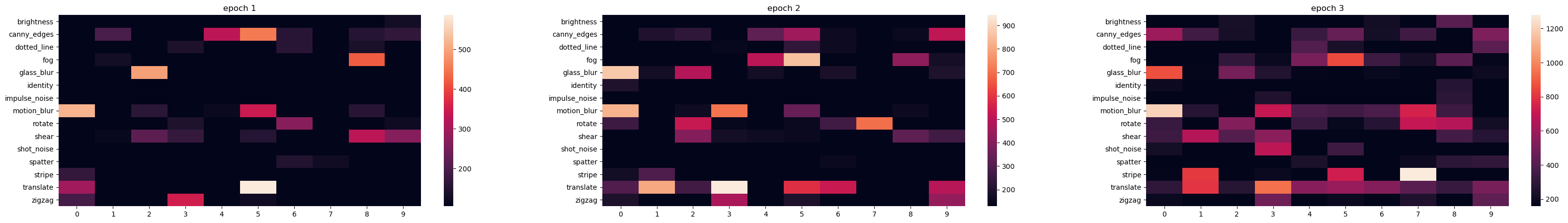}
    \includegraphics[scale=0.22]{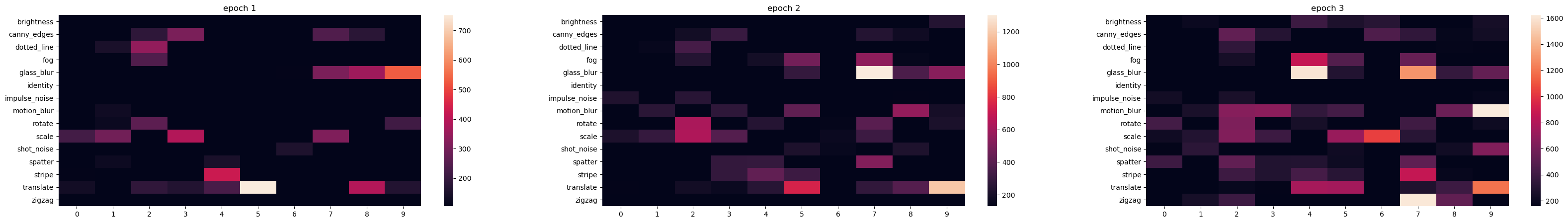}
    \caption{ \textbf{top: sample distribution for target task as \textit{scale\_5}, bottom: sample distribution for target task as shear\_9}
    \small We show the sample distribution at each epoch 1,2,3.
    }
    \label{fig: bad weight}
\end{figure}

\begin{figure}[H]
    \centering
    \includegraphics[scale=0.4]{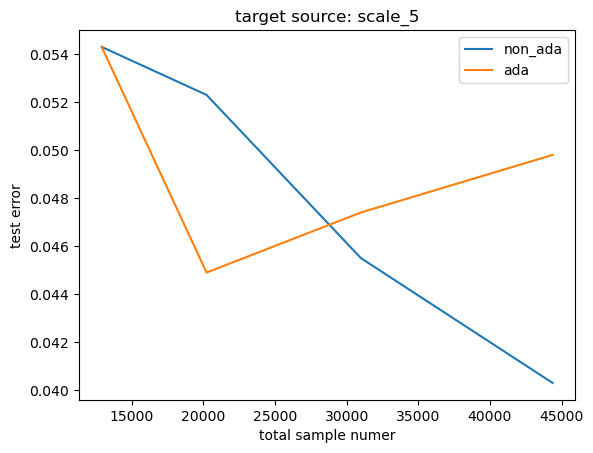}
    \includegraphics[scale=0.4]{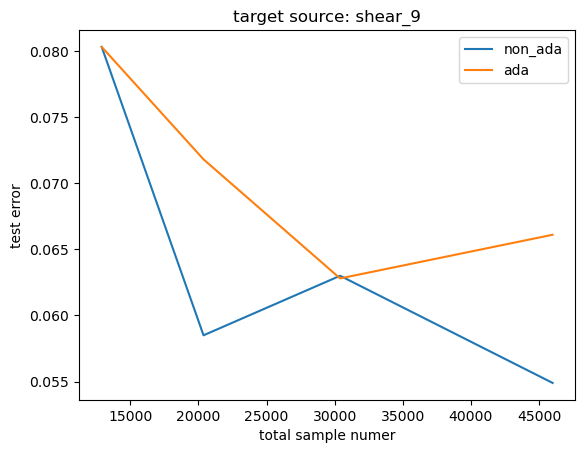}
    \caption{ \textbf{Test error change after each epoch for target task as \textit{scale\_5} and \textit{shear\_9}}
    }
    \label{fig: bad test error}
\end{figure}

\subsubsection{More good sample distribution examples}
\label{subsec: good example cnn}
\begin{figure}[]
    \centering
    \includegraphics[scale=0.22]{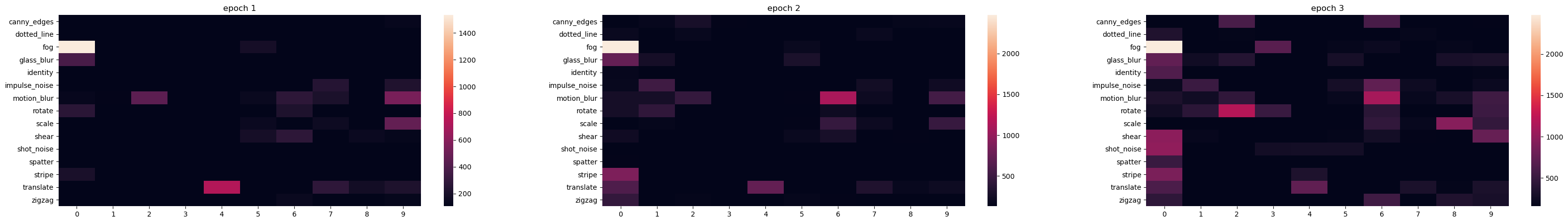}
    \includegraphics[scale=0.22]{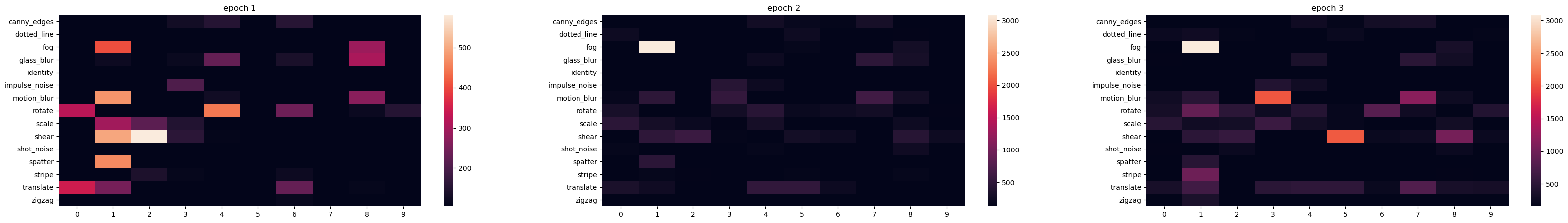}
    \includegraphics[scale=0.22]{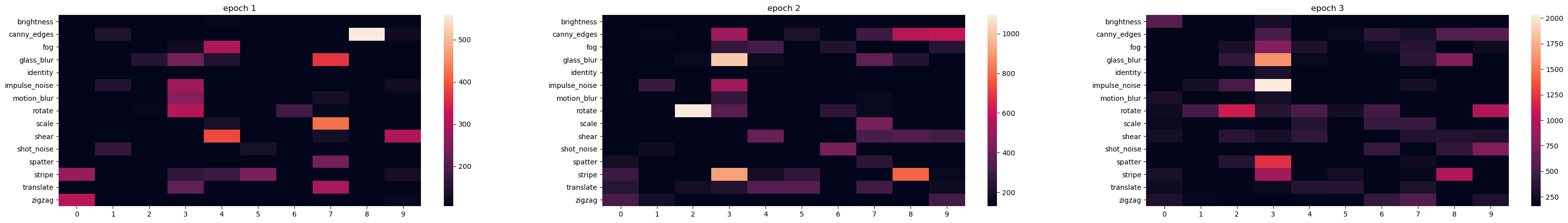}
    \includegraphics[scale=0.22]{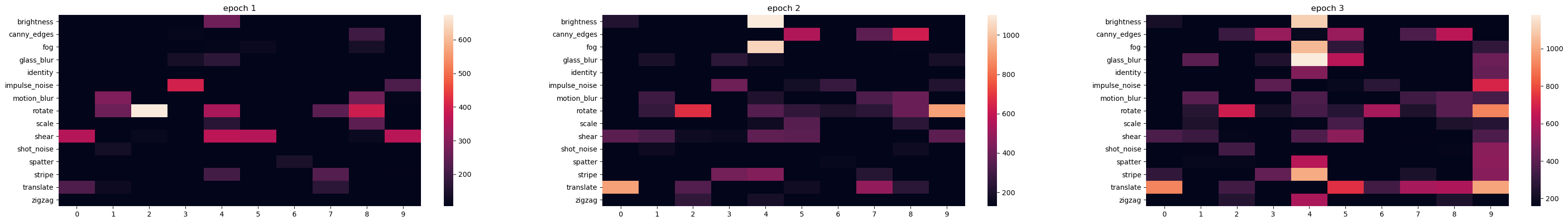}
    \includegraphics[scale=0.22]{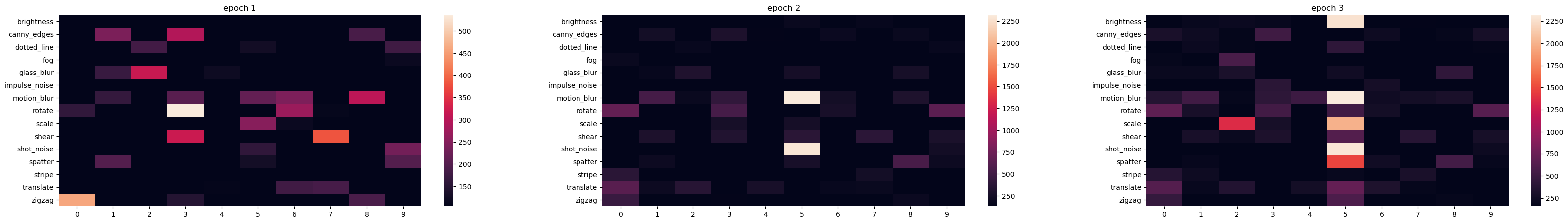}
    \includegraphics[scale=0.22]{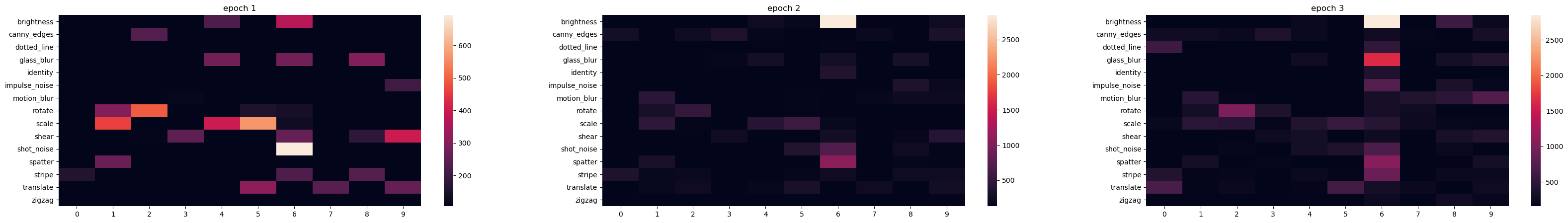}
    \includegraphics[scale=0.22]{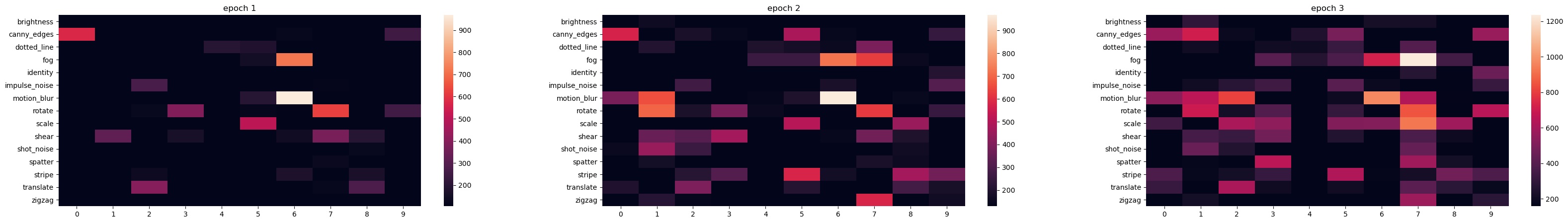}
    \includegraphics[scale=0.22]{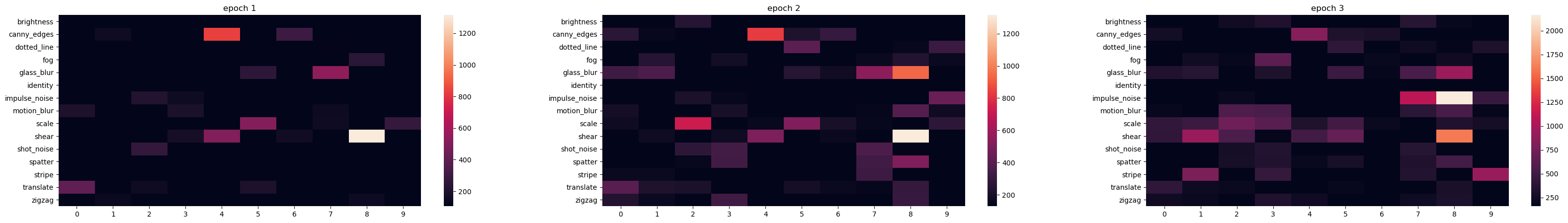}
    \includegraphics[scale=0.22]{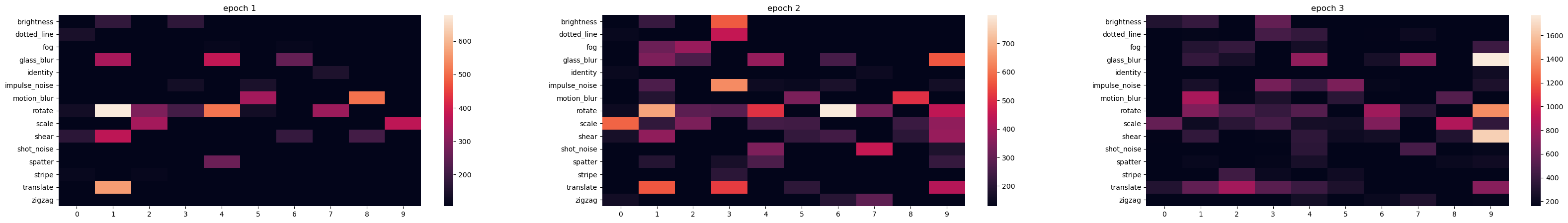}
    \caption{\textbf{Good sample distribution.}  We show the sample distribution at each epoch 1,2,3. \textbf{From top to bottom:} \textit{brightness\_0},\textit{brightness\_1},\textit{dotted\_line\_3},\textit{dotted\_line\_4},\textit{identity\_5}, \textit{fog\_6},\textit{glass\_blur\_7},\textit{rotate\_8},\textit{canny\_edges\_9}}
\end{figure}


\end{document}